\documentclass[10pt,letterpaper]{article}
 
\usepackage[
  letterpaper,
  textwidth=5.5in,
  textheight=9in,
  left=1.5in,
  top=1in
]{geometry}
 
\usepackage{amsmath,amssymb,amsfonts,amsthm}
\usepackage{mathtools}             
 
\usepackage{newtxtext,newtxmath}
 
\usepackage[utf8]{inputenc}        
\usepackage[T1]{fontenc}           
\usepackage{microtype}             
\usepackage{bm}                    
\usepackage{nicefrac}              
\usepackage{graphicx}              
\usepackage{booktabs}              
\usepackage{dcolumn}               
\usepackage{subcaption}            
\usepackage{tikz}                  
\usepackage{algorithm}             
\usepackage{algpseudocode}         
\usepackage{authblk}               
\usepackage{nameref}
\usepackage{url}
\usepackage{xcolor}
\usepackage[numbers,sort&compress]{natbib}  
 
\usepackage[colorlinks=true,allcolors=blue!55!black]{hyperref}
\usepackage[capitalize,noabbrev]{cleveref}   

\makeatletter
\renewcommand\normalsize{%
  \@setfontsize\normalsize{10pt}{12pt}%
  \abovedisplayskip      7pt plus 2pt minus 5pt%
  \belowdisplayskip      \abovedisplayskip%
  \abovedisplayshortskip 0pt plus 3pt%
  \belowdisplayshortskip 4pt plus 3pt minus 3pt%
}
\normalsize
\makeatother
 
\theoremstyle{plain}
\newtheorem{theorem}{Theorem}
\newtheorem{lemma}[theorem]{Lemma}

\newtheorem{proposition}{Proposition}

\theoremstyle{definition}
\newtheorem{definition}{Definition}[section]
\newtheorem{assumption}{Assumption}
\newtheorem{remark}{Remark}
 
\newenvironment{proofsketch}{%
  \proof}{\endproof}
 
\newcommand{\cind}{\mathrel{\perp\mspace{-9mu}\perp}}   
\newcommand{\notcind}{\mathrel{\,\not\!\cind}}          

\let\emptyset\varnothing
 
\title{Beyond Additivity: Causal Discovery in Location-Scale
  Noise Models with Hidden Variables}
\author[1]{Mariyam Khan\footnote{Work conducted while a student trainee at RIKEN AIP}}
\author[2,3,4]{Shohei Shimizu}
\author[2,3,4]{Thong Pham\footnote{Corresponding author: \href{mailto:thong-pham@ds.sanken.osaka-u.ac.jp}{thong-pham@ds.sanken.osaka-u.ac.jp}}}
\affil[1]{University of Bergen}
\affil[2]{The University of Osaka}
\affil[3]{Shiga University}
\affil[4]{RIKEN AIP}
\date{\today}
 
\begin{document}
\maketitle

\begin{abstract}
We study causal discovery from observational data when some variables are hidden and the data-generating process follows a location-scale noise model (LSNM). Existing methods that handle hidden confounders typically assume additive noise, but in practice, causes often modulate not just the mean but also the variance of their effects. We prove that acyclic directed mixed graphs (ADMGs) satisfying a bow-free condition are identifiable under LSNM with hidden variables, establishing the first identifiability result for causally insufficient models beyond noise additivity. We further provide sufficient conditions for identifying causal direction even when the bow-free assumption is violated. Our two-stage algorithm, LSNM-UV, is sound and complete, and experiments demonstrate improved performance over additive baselines on heteroscedastic data.
\end{abstract}

\section{Introduction}

Constraint-based causal discovery uses conditional independence tests to
recover causal structure only up to Markov equivalence~\citep{spirtes2000,pearl2009}.
In the fully observed setting, PC returns a CPDAG representing a Markov
equivalence class of DAGs. With latent confounding and selection bias, FCI
returns a PAG representing an equivalence class of maximal ancestral graphs
(MAGs), and therefore cannot distinguish among MAGs in the same
class~\citep{richardson2002,zhang2008}.

A complementary line of work exploits \emph{functional model assumptions} to
break the symmetry between causal directions. LiNGAM~\citep{shimizu2006}
showed that linear acyclic models with non-Gaussian errors are identifiable
beyond Markov equivalence. Additive noise models (ANMs)~\citep{hoyer2009,peters2014}
and causal additive models (CAM)~\citep{buhlmann2014cam} extended this idea to
nonlinear equations of the form
\[
x_i = f_i(\mathrm{PA}_i) + \epsilon_i,
\]
and post-nonlinear models~\citep{zhang2009pnl} further broadened the functional
class. In their standard formulations, however, these methods assume
\emph{causal sufficiency}.

In practice, causal sufficiency rarely holds. Existing functional-model approaches to hidden variables have largely focused on linear non-Gaussian or additive-noise mechanisms. \citet{hoyer2008latent} generalized LiNGAM to latent settings via overcomplete ICA but retained linearity. For the nonlinear additive case, \citet{maeda21a} introduced CAM-UV, which uses residual independence tests to identify those parent-child relations among observed variables that are identifiable under the model, while flagging pairs affected by unobserved backdoor paths (UBPs) or unobserved causal paths (UCPs). The natural target representation in this latent-variable setting is the acyclic directed mixed graph (ADMG)~\citep{richardson2003}, in which directed edges encode direct causal relations among observed variables and bidirected edges encode latent common causes. \citet{pham2025} sharpened these results by showing that certain causal directions remain identifiable even in the presence of UBPs or UCPs, and \citet{ashman2023causal} developed neural ADMG learning under bow-free nonlinear additive-noise models. In all of these approaches, however, the noise enters additively: the conditional noise scale is not allowed to depend on the parents.

Yet in many real-world processes, causes modulate not only the conditional
location but also the \emph{variance}. Location-scale noise models (LSNMs),
\[
x_i = f_i(\mathrm{PA}_i) + g_i(\mathrm{PA}_i)\epsilon_i,
\]
capture this by allowing the noise scale to depend on the parents. LSNM
identifiability has been established in the bivariate cause-effect
setting~\citep{immer2023,sun2023lsnm} and extended to multivariate DAG
discovery under heteroscedastic-noise assumptions~\citep{yin2024,lin2025skewscore}.
However, existing LSNM graph-identifiability results primarily address
causally sufficient settings. A general theory for identifying observed-variable
causal structure under location-scale noise with hidden variables remains open.

We answer this affirmatively. The central observation is that the two-level location-scale structure allows every hidden-variable effect on an observed variable $x_i$ to be absorbed into an \emph{effective noise} term $\eta_i$, reducing the latent-variable problem to a noise-dependence testing problem on the observed margin. Concretely, $\eta_i \cind \eta_j$ if and only if no UBP or UCP connects $x_i$ and $x_j$ (Lemma~\ref{lemma:independence_patterns}). Our contributions are:
\begin{itemize}
    \item \textbf{Model and noise reduction.} We formulate LSNM-UV, a two-level location-scale noise model with observed and hidden variables (Equation~\eqref{eq: LSNM-UV}), and show that it reduces to an observed-variable model with dependent effective noise whose independence structure encodes UBPs and UCPs (Lemma~\ref{lemma:independence_patterns}).
    \item \textbf{Residual characterization.} We characterize how location-scale residual independence patterns distinguish bidirected edges, directed edges, and non-edges in the projected ADMG (Proposition~\ref{prop:residual_independence_pattern}).
    \item \textbf{Identifiability.} We prove that bow-free ADMGs are identifiable under LSNM-UV (Theorem~\ref{thm:ADMG_identifiability}); this is, to our knowledge, the first identifiability result for causally insufficient models beyond noise additivity.
    \item \textbf{Beyond bow-free.} We give sufficient conditions for identifying causal direction even in the presence of bows, using visible parents as a side channel (Lemma~\ref{lem:checkonpath}).
    \item \textbf{Algorithm and experiments.} We propose LSNM-UV, a two-stage algorithm that is provably sound and complete (Theorem~\ref{thm:soundness_completeness}), and show empirically that it outperforms additive-noise baselines on heteroscedastic data.
\end{itemize}
\section{Preliminaries}\label{sec:prelim}
\subsection{Identifiability and oracles}\label{sec:prelim_identifiability}
Let $\mathcal{G}$ denote the class of graphs under consideration. Collect all parameters, functions, and exogenous noises into a single (potentially infinite-dimensional) object $\theta$, and let $\Theta_{G}$ denote the parameter space induced by the model assumptions for a given graph $G$. Let $P_X$ denote the marginal distribution over the observed variables $X$; we write $P_{G,\theta}$ when emphasizing the dependence on $G$ and $\theta$. We adopt the following standard definition of identifiability~\cite{Peters11UAI}.
\begin{definition}[Identifiability] \label{def:Identifiability} A graph function $T$ is identifiable if and only if $\ \forall G,G' \in \mathcal{G},\  \forall \theta' \in \Theta_{G'}, \forall \theta \in \Theta_{G}$:
$$P_{G,\theta} = P_{G',\theta'} \implies T(G) = T(G').$$
\end{definition}
An oracle $W$ is a deterministic function $W(P_X,q)$ of the observed distribution $P_X$ and the query $q$.

\begin{definition}[Oracle identifiability] \label{def:Oracle-Identifiability}
$T$ is oracle identifiable with respect to an oracle $W$ if and only if $\ \forall G,G' \in \mathcal{G},\  \forall \theta' \in \Theta_{G'},\ \forall \theta \in \Theta_{G}$:
$$W(P_{G,\theta},q) = W(P_{G',\theta'},q) \ \forall q \implies T(G) = T(G').$$
\end{definition}
\begin{lemma}\label{lemma:oracle_implies_ident}
Oracle identifiability implies identifiability.
\end{lemma}
\begin{proof}
Suppose $T$ is oracle identifiable w.r.t.\ $W$, and let $P_{G,\theta} = P_{G',\theta'}$. Since $W$ is a deterministic function of $P_X$, we have $W(P_{G,\theta}, q) = W(P_{G',\theta'}, q)$ for all $q$. By oracle identifiability, $T(G) = T(G')$.
\end{proof}
\noindent Proving identifiability directly requires showing that the observed data distribution determines $T(G)$ uniquely across all graphs and parameters, which is often intractable. The oracle provides an easier route; we introduce an oracle $W$ that takes as input the observed data distribution $P_X$ and a query $q$, and returns an answer, without access to the true graph or hidden variables. Oracle identifiability (Definition~\ref{def:Oracle-Identifiability}) says that if two setups give the same oracle answers for every query, they must agree on $T$. Since $W$ depends only on $P_X$, any two setups producing the same data distribution automatically produce the same oracle answers, so oracle identifiability implies identifiability (Lemma~\ref{lemma:oracle_implies_ident}). Thus, it suffices to construct an oracle whose answers determine $T(G)$.

\subsection{Unobserved backdoor paths and unobserved causal paths}
Hidden variables can influence observed-variable relationships through two mechanisms, first defined by \citet{maeda21a} and \citet{pham2025} for causal additive models. We state them here in our notation.
\begin{definition}[Unobserved Causal Path]
\label{def:ucp}
We say there is an \emph{unobserved causal path} (UCP) from $x_j$ to $x_i$ when $x_j$ is an ancestor of some unobserved variable $u_k \notin X$ that is itself a direct parent of $x_i$, so that a directed path $x_j \to \cdots \to u_k \to x_i$ exists through unobserved intermediaries. We say a UCP exists \emph{between} $x_i$ and $x_j$ if a UCP exists in either direction.
\end{definition}

\begin{definition}[Unobserved Backdoor Path]
\label{def:ubp}
We say there is an \emph{unobserved backdoor path} (UBP) between $x_i$ and $x_j$ when there exist unobserved parents $u_k \notin X$ of $x_i$ and $u_l \notin X$ of $x_j$ that share a common ancestor $v$ in the DAG, forming the path $x_i \leftarrow u_k \leftarrow \cdots \leftarrow v \to \cdots \to u_l \to x_j$. The cases $v = u_k$, $v = u_l$, or $u_k = v = u_l$ are all permitted.
\end{definition}



\subsection{Acyclic Directed Mixed Graphs}
An acyclic directed mixed graph (ADMG) on the observed variables $X$ is a mixed graph containing directed edges $x_i \to x_j$ and bidirected edges $x_i \leftrightarrow x_j$, with no directed cycles. We represent an ADMG by a pair of matrices $(A,B)$, where $A(i,j) = 1$ iff $x_j \to x_i$ (i.e., $x_j$ is a parent of $x_i$), and $B(i,j) = 1$ iff $x_i \leftrightarrow x_j$ (i.e., there is a UBP or UCP between $x_i$ and $x_j$). The bidirected matrix $B$ is symmetric; both $A$ and $B$ have zero diagonal.

\begin{definition}[Bow-free ADMG]
\label{def:bowfree}
An ADMG $(A,B)$ is bow-free if no pair of observed variables is connected by both a directed edge and a bidirected edge; equivalently,
\[
A(i,j)\,B(i,j) = 0 \qquad \forall\, i \neq j.
\]
\end{definition}

\begin{figure}[H]
\centering
\begin{minipage}[t]{0.46\textwidth}
\centering
\textbf{(A) Full latent DAG $G$}\\[6pt]
\begin{tikzpicture}[
  obs/.style={circle, draw, thick, minimum size=20pt, inner sep=0pt, font=\small},
  hid/.style={circle, draw, dashed, thick, minimum size=20pt, inner sep=0pt,
              font=\small, fill=gray!10},
  arr/.style={->, >=stealth, semithick},
  scale=0.55,
]
\node[hid] (u1) at (1.75, 4.5) {$u_1$};
\node[obs] (x1) at (0, 2.8)    {$x_1$};
\node[obs] (x2) at (3.5, 2.8)  {$x_2$};
\node[obs] (x3) at (0, 0.8)    {$x_3$};
\node[hid] (y1) at (2.8, 1.4)  {$y_1$};
\node[obs] (x4) at (3.5, -0.2) {$x_4$};
\node[obs] (x5) at (-1.2, -0.6){$x_5$};

\draw[arr] (u1) -- (x1);
\draw[arr] (u1) -- (x2);
\draw[arr] (x1) -- (x3);
\draw[arr] (x1) -- (y1);
\draw[arr] (x3) -- (x4);
\draw[arr] (x3) -- (x5);
\draw[arr] (y1) -- (x4);
\end{tikzpicture}
\end{minipage}%
\hfill
\begin{minipage}[t]{0.46\textwidth}
\centering
\textbf{(B) Projected bow-free ADMG}\\[6pt]
\begin{tikzpicture}[
  obs/.style={circle, draw, thick, minimum size=20pt, inner sep=0pt, font=\small},
  dir/.style={->, >=stealth, semithick, blue!70!black},
  bidir/.style={<->, >=stealth, semithick, red!70!black, dashed},
  scale=0.55,
]
\node[obs] (x1) at (0, 2.8)    {$x_1$};
\node[obs] (x2) at (3.5, 2.8)  {$x_2$};
\node[obs] (x3) at (0, 0.8)    {$x_3$};
\node[obs] (x4) at (3.5, -0.2) {$x_4$};
\node[obs] (x5) at (-1.2, -0.6){$x_5$};

\draw[dir] (x1) -- (x3);
\draw[dir] (x3) -- (x4);
\draw[dir] (x3) -- (x5);

\draw[bidir, bend left=22] (x1) to (x2);
\draw[bidir, bend right=18] (x1) to (x4);
\draw[bidir, bend left=22] (x2) to (x4);
\end{tikzpicture}
\end{minipage}

\caption{Illustrative example with observed $X=\{x_1,\dots,x_5\}$ and hidden $\{u_1, y_1\}$ (dashed nodes).
\textbf{(A)}~Full latent DAG.
\textbf{(B)}~Projected bow-free ADMG where blue solid   edges are directed and red dashed edges are bidirected (UCP or UBP). Here $(x_1, x_2)$ is a \emph{UBP} via the hidden common cause $x_1 \leftarrow u_1 \to x_2$; $(x_1, x_4)$ is a \emph{UCP} via the hidden intermediate $x_1 \to y_1 \to x_4$;
and $(x_2, x_4)$ is a \emph{UBP} via $x_2 \leftarrow u_1 \to x_1 \to y_1 \to x_4$, where the path passes through observed $x_1$ but the endpoints $u_1$ and $y_1$ are both hidden.
In contrast, $(x_4, x_5)$ has \emph{no} bidirected edge  because the parent $x_3$ is observed;
$(x_2, x_3)$ has \emph{no} bidirected edge since $x_3$'s parent $x_1$ is observed;
and $(x_1, x_5)$ is connected only by the fully observed path $x_1 \to x_3 \to x_5$.}
\label{fig:ubp_ucp}
\end{figure}
\section{The Causal Model}
Let $X = \{x_i\}$ and $U = \{u_k\}$ denote the sets of observed and unobserved variables, respectively. Let $G = (V, E)$ be a DAG with vertex set $V = X \cup U$ and directed edge set $E \subseteq V \times V$, where $(v_j, v_i) \in E$ denotes $v_j \to v_i$. For each variable $v_i \in V$, define:
\begin{itemize}
\item $K_i = \{x_j \in X \mid (x_j, v_i) \in E\}$: the set of observed parents of $v_i$,
\item $Q_i = \{u_k \in U \mid (u_k, v_i) \in E\}$: the set of unobserved parents of $v_i$.
\end{itemize}
The data generation model is a two-level location-scale noise model:
\begin{equation}
v_i = f_{i}^1(K_i) + g_{i}^1(K_i)\,f_{i}^2(Q_i) +  g_{i}^1(K_i)\, g_{i}^2(Q_i)\,\epsilon_i, \label{eq: LSNM-UV}
\end{equation}
where $f_{i}^1, g_{i}^1$ are nonlinear functions of the observed parents, $f_{i}^2, g_{i}^2$ are nonlinear functions of the unobserved parents, and $\epsilon_i$ is the exogenous noise at $v_i$. The noises $\{\epsilon_i\}$ are mutually independent. When $Q_i = \emptyset$, the hidden-parent functions reduce to constants, and Equation \eqref{eq: LSNM-UV} becomes a standard location-scale noise model $v_i = f_i^1(K_i) + g_i^1(K_i)\,\epsilon_i$. When additionally $g_i^1 \equiv 1$ for all $i$, the model reduces to the causal additive model of \citet{maeda21a}.

\begin{assumption}[Causal Faithfulness Condition (CFC)]\label{assumption:CFC}
Any conditional independence on $V = X \cup U$ that is not entailed by the d-separation criterion on $G$ does not hold~\citep{spirtes2000,pearl2009}.
\end{assumption}


\section{Identifiability result}
The key property of Equation~\eqref{eq: LSNM-UV} is that, defining $\eta_i \coloneqq f_{i}^2(Q_i) + g_{i}^2(Q_i)\epsilon_i$, it reduces to a location-scale model on observed variables with \emph{dependent} noise:
\begin{align}
v_i &= f_{i}^1(K_i) + g_{i}^1(K_i)\!\left(f_{i}^2(Q_i) + g_{i}^2(Q_i)\epsilon_i\right) = f_{i}^1(K_i) + g_{i}^1(K_i)\,\eta_i. \label{eq: LSNM-dependent-noise}
\end{align}

\begin{lemma}\label{lemma:independence_patterns}
Under Assumption~\ref{assumption:residual_faithfulness} (residual faithfulness),
$\eta_i \cind \eta_j 
$ if and only if there is no UBP or UCP between $x_i$ and $x_j$.
\end{lemma}
\begin{proofsketch}
For each observed variable $x_i$, the effective noise $\eta_i$ is a deterministic function of the noise set $N_i = \{ \epsilon_l: v_l \in \mathrm{an}(Q_i) \cup Q_i \cup\{x_i\}\}$. We show that if there is no UBP and no UCP between $x_i$ and $x_j$, then $N_i \cap N_j = \emptyset$, any shared $\epsilon_l$ would imply a UCP or UBP. Since $\epsilon_l$'s are mutually independent, this would imply $ \eta_i \cind \eta_j$. For the other direction we show that if a UBP or UCP exists between $x_i$ and $x_j$, then some $\epsilon_a \in N_i \cap N_j$ and hence by the \emph{residual faithfulness condition} (Assumption \ref{assumption:residual_faithfulness}), $ \eta_i \notcind \eta_j$.
\end{proofsketch}
We now instantiate the graph function $T$ from Definition~\ref{def:Identifiability}. Let $T_{ADMG}(G) = (A,B)$ denote the projection of the full DAG $G$ onto the ADMG over $X$, where $A(i,j) = 1$ iff $x_j$ is a parent of $x_i$ in $G$, and $B(i,j) = 1$ iff there exists a UBP or UCP between $x_i$ and $x_j$ in $G$.

Let $S$ be a subset of the index set of $X$ with $i \in S$, and let $h(i,S)$
be any function of the form
$$h(i,S) = \frac{x_i - f(\mathbf{x}_{X \setminus S})}{g(\mathbf{x}_{X \setminus S})}$$
where both the location function $f$ and the scale function $g$ use the same regression set $X \setminus S$. More generally, $f$ and $g$ may exclude different subsets $S_f, S_g \subseteq S$ and all results continue to hold; we use a single $S$ for notational simplicity. Consider a fixed class
$\mathcal{H}$ of such regression functions $h$. We assume
that the data generating process satisfies the following residual
faithfulness condition w.r.t.\ $\mathcal{H}$, analogous
to the assumption used in additive-noise
models~\citep{maeda21a,ashman2023causal,pham2025}.
\begin{assumption}[Residual faithfulness for LSNM]
\label{assumption:residual_faithfulness}
When both $h_1(i, S)$ and $h_2(j, S')$ have terms involving functions of the same exogenous noise $\epsilon_k$, they are mutually dependent:
\[
(\epsilon_k \notcind h_1(i, S)) \;\wedge\; (\epsilon_k \notcind h_2(j, S')) \;\Longrightarrow\; h_1(i, S) \notcind h_2(j, S').
\]
\end{assumption}
The interaction between $\mathcal{H}$ and the observed data gives the following characterization of $T_{ADMG}(G)$.

\begin{proposition}\label{prop:residual_independence_pattern}
Assume the data is generated from the SCM in Equation~(\ref{eq: LSNM-UV}) with the true graph $G$. Let $T_{ADMG}(G) = (A,B)$. Assume some residual faithfulness assumption \ref{assumption:residual_faithfulness}.
\begin{align}
B(i,j) = 1 &\iff \forall\, h_1,h_2\in \mathcal{H}:\; h_1(i, \{i\}) \notcind h_2(j,\{j\}). \label{eq:case1-main} \\
A(i,j) = 1 \land B(i,j) = 0 &\iff \forall\, h_1,h_2\in \mathcal{H}:\; h_1(i,\{i,j\}) \notcind h_2(j,\{j\}), \label{eq:case2a-main} \\
& \qquad \text{and}\; \exists\, h_1,h_2\in \mathcal{H}:\; h_1(i,\{i\}) \cind h_2(j,\{j\}). \label{eq:case2b-main} \\
A(i,j) = 0 \land A(j,i) = 0 \land B(i,j) = 0 &\iff \exists\, h_1,h_2\in \mathcal{H}:\; h_1(i,\{i,j\}) \cind h_2(j,\{i,j\}). \label{eq:case3-main}
\end{align}
\end{proposition}
\begin{proofsketch}
We prove each case by identifying which exogenous noises are 
retained in the residuals and applying the residual faithfulness 
condition (Assumption \ref{assumption:residual_faithfulness}).

\emph{Case 1.} ($\Rightarrow$) If $B(i,j) = 1$, a shared exogenous  noise $\epsilon_l$ enters both $\eta_i$ and $\eta_j$ through hidden paths (Lemma~\ref{lemma:independence_patterns}). Since  regression functions in $h_1$ and $h_2$ use only observed 
variables, $\epsilon_l$ cannot be removed from either residual. 
By residual faithfulness, all residual pairs are dependent.
($\Leftarrow$) If $B(i,j) = 0$, the true structural functions 
yield $h_1(i,\{i\}) = \eta_i$ and $h_2(j,\{j\}) = \eta_j$, 
which are independent by Lemma~\ref{lemma:independence_patterns}.

\emph{Case 2.} ($\Rightarrow$) For~\eqref{eq:case2a-main}: excluding 
$x_j$ from at least one of $f, g$ (required by $S = \{i,j\}$) 
leaves $\epsilon_j$ in $h_1$, while $\epsilon_j$ is irremovable 
from $h_2$ (Lemma~\ref{lemma:noise-irremovability-LSNM-residuals}). 
Faithfulness gives universal dependence.
For~\eqref{eq:case2b-main}: the true functions give 
$\eta_i \cind \eta_j$.
($\Leftarrow$) Independence in~\eqref{eq:case2b-main} directly 
contradicts Case~1, giving $B(i,j) = 0$. If $A(i,j) = 0$, the 
true functions also satisfy $S = \{i,j\}$, yielding 
$\eta_i \cind \eta_j$, contradicting~\eqref{eq:case2a-main}.

\emph{Case 3.} ($\Rightarrow$) Since neither is a parent of the 
other, the true functions satisfy $S = \{i,j\}$ for both, yielding 
$\eta_i \cind \eta_j$.
($\Leftarrow$) If $B(i,j) = 1$, hidden noise makes all residuals 
dependent regardless of $S$ (same argument as Case~1). If 
$A(i,j) = 1$, excluding $x_j$ from $h_1$ retains $\epsilon_j$, 
while $\epsilon_j$ is irremovable from $h_2$. Both contradict the 
assumed independence. Symmetrically $A(j,i) = 0$.
\end{proofsketch}


\begin{figure}[t]
\centering
\includegraphics[width=0.9\textwidth]{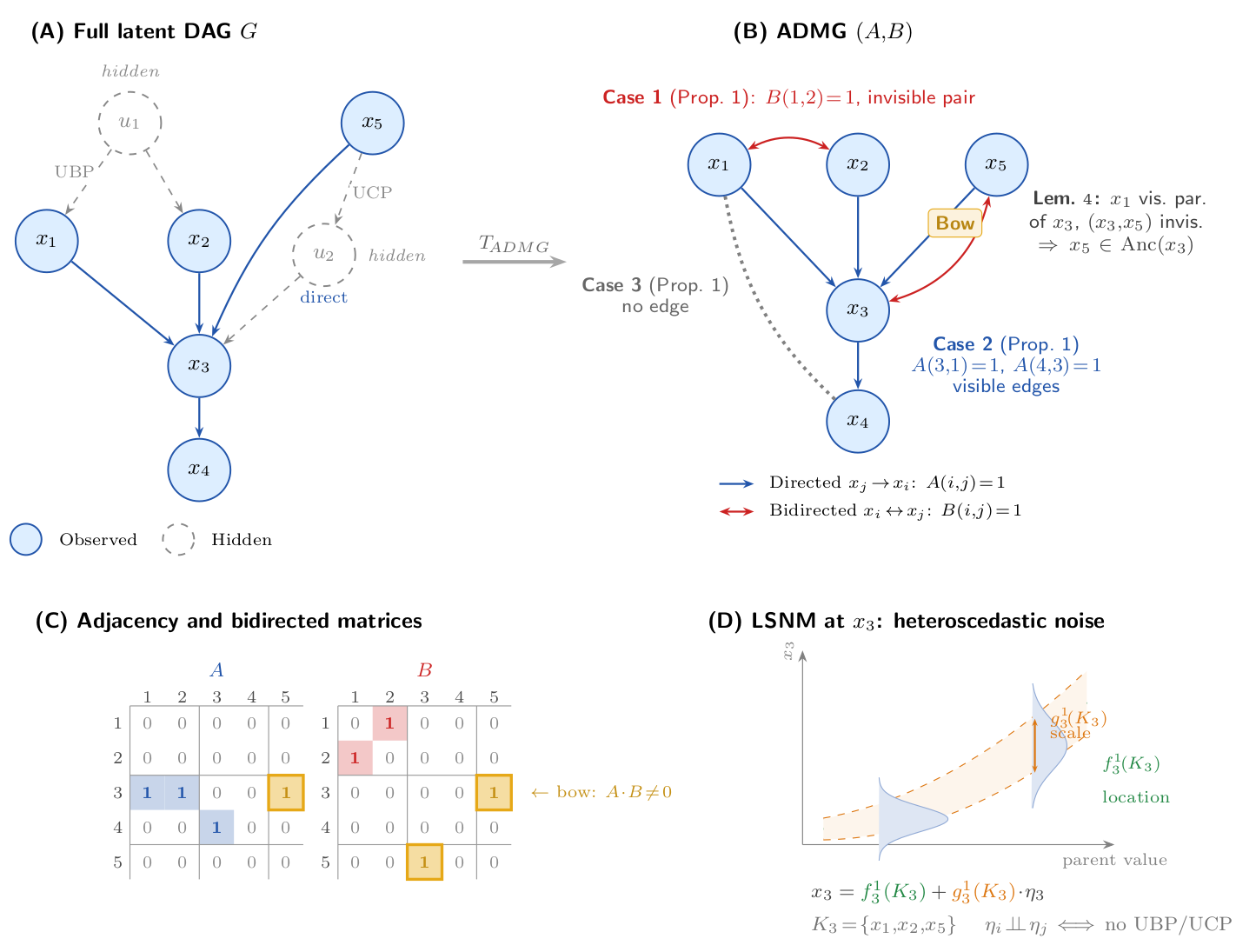}
\caption{Overview of the LSNM-UV framework.
\textbf{(A)}~Full latent DAG with observed variables $\{x_1,\ldots,x_5\}$ and hidden causes $\{u_1, u_2\}$.
\textbf{(B)}~Projected ADMG: directed edges ($\to$), bidirected edges ($\leftrightarrow$), and the three cases of Proposition~\ref{prop:residual_independence_pattern} (invisible pair, visible edge, non-edge).
\textbf{(C)}~Ground-truth adjacency matrix $A$ and bidirected matrix $B$.
\textbf{(D)}~LSNM at $x_3$: parents $K_3$ determine both location $f_3^1$ and scale $g_3^1$; the residual $\eta_3$ isolates the hidden-variable signal used for UBP/UCP detection.}
\label{fig:overview}
\end{figure}

\paragraph{Example for Proposition~\ref{prop:residual_independence_pattern}.}
We illustrate Proposition~\ref{prop:residual_independence_pattern} using the graph in Figure~\ref{fig:overview}. The structural equations (Equation~\ref{eq: LSNM-dependent-noise}) are:
\begin{align*}
x_1 &= \eta_1, &\eta_1 &= f_1^2(u_1) + g_1^2(u_1)\,\epsilon_1 & &(K_1 = \emptyset,\; Q_1 = \{u_1\}) \\
x_2 &= \eta_2, &\eta_2 &= f_2^2(u_1) + g_2^2(u_1)\,\epsilon_2 & &(K_2 = \emptyset,\; Q_2 = \{u_1\}) \\
x_5 &= \epsilon_5 & & & &(K_5 = Q_5 = \emptyset) \\
x_3 &= f_3^1(K_3) + g_3^1(K_3)\,\eta_3,\quad &\eta_3 &= f_3^2(u_2) + g_3^2(u_2)\,\epsilon_3 & &(K_3 = \{x_1,x_2,x_5\},\; Q_3 = \{u_2\}) \\
x_4 &= f_4^1(x_3) + g_4^1(x_3)\,\epsilon_4 & & & &(K_4 = \{x_3\},\; Q_4 = \emptyset)
\end{align*}
Here $u_2$ is a hidden descendant of $x_5$ (creating the UCP $x_5 \!\to\! u_2 \!\to\! x_3$) and $u_1$ is a hidden parent of $x_1$ and $x_2$ (creating the UBP $x_1 \leftarrow u_1 \to x_2$). For each observed variable $x_i$, the effective noise $\eta_i$ is a deterministic function of the noise set $N_i = \{ \epsilon_l: v_l \in \mathrm{an}(Q_i) \cup Q_i \cup\{x_i\}\}$. 
The noise sets here are $N_1 = \{\epsilon_{u_1}, \epsilon_1\}$, $N_2 = \{\epsilon_{u_1}, \epsilon_2\}$, $N_5 = \{\epsilon_5\}$, $N_3 = \{\epsilon_5, \epsilon_{u_2}, \epsilon_3\}$, $N_4 = \{\epsilon_4\}$.

With $S = \{3, 1\}$, the residual is $h(3, \{3,1\}) = (x_3 - f(\mathbf{x}_{X \setminus \{3,1\}}))\, /\, g(\mathbf{x}_{X \setminus \{3,1\}})$, where $f$ and $g$ range over all functions of variables in the regression set $X \setminus \{3,1\} = \{x_2, x_4, x_5\}$.


\medskip
\noindent\textbf{Example 1 (Invisible pair): $(x_1, x_2)$, \; $B(1,2) = 1$.}
Claim: $\forall\, h_1, h_2$: $h_1(1, \{1\}) \notcind h_2(2, \{2\})$.

With $S = \{1\}$, both $f$ and $g$ use the observed variables $X \setminus \{1\} = \{x_2, x_3, x_4, x_5\}$. Since $x_1 = \eta_1$ depends on $\epsilon_{u_1}$ through hidden $u_1$, no choice of $f, g$ removes $\epsilon_{u_1}$ from $h_1$. By symmetry, $\epsilon_{u_1}$ is irremovable from $h_2(2, \{2\})$. The shared noise gives $h_1 \notcind h_2$ for all $h_1, h_2$ (Assumption~\ref{assumption:residual_faithfulness}). In particular, no $h_1, h_2$ achieves $h_1(1,\{1\}) \cind h_2(2,\{2\})$, so the test of Case~2 (Equation \eqref{eq:case2b-main}) fails.

\medskip
\noindent\textbf{Example 2 (Visible parent): $x_1 \to x_3$, \; $A(3,1) = 1, \; B(3,1) = 0$.} For \emph{Equation~\eqref{eq:case2a-main}}, we have $\forall\, h_1, h_2$: $h_1(3, \{3, 1\}) \notcind h_2(1, \{1\})$. With $S = \{3, 1\}$, both $f$ and $g$ use $\{x_2, x_4, x_5\}$ and neither has access to $x_1$. Since $x_1$ is a true parent appearing nonlinearly in both $f_3^1$ and $g_3^1$, no $f, g$ on $\{x_2, x_4, x_5\}$ can cancel $x_1$'s contribution (Lemma~\ref{lemma:non_decomp}). Therefore $\epsilon_1$ remains in $h_1(3, \{3,1\})$ for all $f, g$. Since $\epsilon_1$ is also irremovable from $h_2(1,\{1\})$ (Lemma~\ref{lemma:noise-irremovability-LSNM-residuals}), we get universal dependence. For \emph{Equation~\eqref{eq:case2b-main}}, we show $\exists\, h_1, h_2$: $h_1(3, \{3\}) \cind h_2(1, \{1\})$. When $S = \{3\}$: the regression set is $\{x_1, x_2, x_4, x_5\}$, so $x_1$ is available to both $f$ and $g$. With the true structural functions:
\[
h_1^*(3,\{3\}) = \frac{x_3 - f_3^1(x_1,x_2,x_5)}{g_3^1(x_1,x_2,x_5)} = \eta_3 = f_3^2(u_2) + g_3^2(u_2)\,\epsilon_3.
\]
Similarly, $h_2^*(1, \{1\}) = x_1 = \eta_1$. The noise sets of $\eta_3$ and $\eta_1$ are $\{\epsilon_5, \epsilon_{u_2}, \epsilon_3\}$ and $\{\epsilon_{u_1}, \epsilon_1\}$, which are disjoint ($B(3,1) = 0$), so $\eta_3 \cind \eta_1$. This independence is also the negation of Case~1's Equation~\eqref{eq:case1-main}, confirming $B(3,1) = 0$.

More examples on \emph{Case 3 (No edge)} $(x_1, x_5)$ and \emph{indirect path} $(x_1, x_4)$, can be found in the appendix.
Define the oracle $W_{RI}$ with query $q = (i,j)$ as follows. When there is a UBP or UCP between $x_i$ and $x_j$, the oracle returns $B(i,j) = 1$. When there is no UBP or UCP, the oracle returns $B(i,j) = 0$ and additionally returns the values of $A(i,j)$ and $A(j,i)$.
Our main result is:
\begin{theorem}\label{thm:ADMG_identifiability}
Let 
$$
\mathcal{G}_{full} = \{G : \text{DAG and } T_{ADMG}(G) \text{ is bow-free}\}.
$$ 
$T_{ADMG}$ is identifiable over $\mathcal{G}_{full}$.
\end{theorem}
\begin{proofsketch}
Based on Proposition~\ref{prop:residual_independence_pattern}, $W_{RI}$ is a function of the observed distribution, and $T_{ADMG}$ is identifiable w.r.t. the oracle $W_{RI}$ over $\mathcal{G}_{full}$. Therefore, $T_{ADMG}$ is also identifiable over $\mathcal{G}_{full}$.     
\end{proofsketch}
Beyond bow-free ADMGs, we provide sufficient conditions to identify causal direction even in the presence of bows. The following result identifies observed variables that lie on invisible paths between a visible parent pair, without requiring the bow-free assumption (see Remark~\ref{rem:no_subset_visibility} for comparison with \cite{pham2025}).
\begin{lemma}[Identifying ancestors on invisible paths]
\label{lem:checkonpath}
Consider distinct $x_i, x_j, x_k \in X$.
Suppose:
\begin{enumerate}
  \item $x_j$ is a visible parent of $x_i$
    (Definition~\ref{def:vis_parent}): $A(i,j) = 1$ and $B(i,j) = 0$.
  \item The pair $(x_i, x_k)$ is invisible
    (Definition~\ref{def:invisible}):
    $\forall\, h_1, h_2 \in \mathcal{H}$:
    $h_1(i, \{i\}) \notcind h_2(k, \{k\})$.
\end{enumerate}
Under Assumptions~\ref{assumption:CFC}
and~\ref{assumption:residual_faithfulness}, suppose further that the following holds:
\begin{equation}\label{eq:cop-necessary1}
  \forall\, h_1, h_2 \in \mathcal{H} :\quad
  h_1(i, \{i, k\})
  \;\notcind\;
  h_2(j, \{j, k\}).
\end{equation}

\noindent Then $x_k$ is an ancestor of $x_i$.
\end{lemma}
\begin{proofsketch}
We prove the contrapositive: assume $x_k \notin \mathrm{Anc}(x_i)$.
Since parents are ancestors, $x_k \notin K_i$.
Since $x_j \to x_i$ implies $\mathrm{Anc}(x_j) \subseteq \mathrm{Anc}(x_i)$, we also have $x_k \notin K_j$.
Choosing the true structural functions yields $h_1^*(i,\{i,k\}) = \eta_i$ and $h_2^*(j,\{j,k\}) = \eta_j$, since neither $K_i$ nor $K_j$ contains $x_k$.
Since $B(i,j) = 0$, by the noise-set argument in Lemma~\ref{lemma:independence_patterns}: the noise sets $N_i$ and $N_j$ are disjoint, so $\eta_i \cind \eta_j$, contradicting Equation~\eqref{eq:cop-necessary1}. Full proof in Appendix~\ref{sec:checkonpath}.
\end{proofsketch}
\noindent\textbf{Example 3 (Bow at $(x_3, x_5)$), \; $A(3,5) = B(3,5) = 1$.} In Figure~\ref{fig:overview}(B), $x_5$ is a direct parent of $x_3$ \emph{and} there is a UCP $x_5 \!\to\! u_2 \!\to\! x_3$. Proposition~\ref{prop:residual_independence_pattern} alone cannot determine $A(3,5)$ when $B(3,5) = 1$, since in Case~1 Equation~\eqref{eq:case1-main} $\epsilon_5 \in N_3 \cap N_5$ is irremovable because of the UCP. Also Case~2 Equation~\eqref{eq:case2b-main} also fails because again $\eta_3$ depends on $\epsilon_5$ through the UCP regardless of the regression set chosen. However, Lemma~\ref{lem:checkonpath} recovers the causal direction by taking $x_1$ as a visible parent of $x_3$ ($A(3,1) = 1$, $B(3,1) = 0$) and $(x_3, x_5)$ as the invisible pair, the dependence in Equation~\eqref{eq:cop-necessary1} holds because excluding $x_5$ from the regression for $x_3$ leaves $x_5$'s contribution irremovable, identifying $x_5$ as an ancestor of $x_3$. The bow-free assumption (Definition~\ref{def:bowfree}) is still needed for full ADMG recovery (Theorem~\ref{thm:ADMG_identifiability}), but Lemma~\ref{lem:checkonpath} extends partial identifiability beyond the bow-free case.
\section{Search method}
\begin{algorithm}[H]
\caption{LSNM-UV}
\label{alg:main}
\begin{algorithmic}[1]
\Require Data matrix $\mathbf{X}$ for $p$ observed variables, max parents $d$, significance level $\alpha$
\Ensure Adjacency matrix $A$
\State \textbf{Stage 1:} $A \leftarrow \text{LSNM-UV-Base}(\mathbf{X}, d, \alpha)$ \hfill\emph{[Algorithm~\ref{alg:lsnm-uv-base}]}

\medskip
\State \textbf{Stage 2:} Re-examine invisible pairs
\State $S \leftarrow \{(i,j) \mid A(i,j) = A(j,i) = \text{NaN}\}$
\For{each $(i,j) \in S$ with $i < j$}
    \State \textsc{checkVisible}$(i, j)$ \hfill\emph{[Algorithm~\ref{alg:checkvisible}]}
\EndFor

\State \Return $A$
\end{algorithmic}
\end{algorithm}
For the output of the algorithm, we would have $A(i,j) = 1$, then $x_j$ is a parent of $x_i$; $A(i,j) = 0$ then  $x_j$ is not a parent of $x_i$ (non-edge if $A(j,i) = 0$ as well) and $A(i,j) = \text{NaN}$, then UBP or UCP exists between $x_i$ and $x_j$.

\begin{assumption}\label{ass:HSIC}
For any $x_i, x_j \in X$ and any valid sets $S_1, S_2$, let $h_1(i,S_1)$ and $h_2(j,S_2)$ be residuals computed using GAMLSS regression functions fitted from data and let $e = p\text{-HSIC}(h_1(i,S_1), h_2(j,S_2))$. For a given significance level $\alpha$, the following equation holds
\[
\exists\, h_1, h_2 \in \mathcal{H} :\; h_1(i, S_1) \cind h_2(j, S_2) \quad \iff \quad e > \alpha.
\]
That is, independence-achieving functions exist within $\mathcal{H}$ if and only if the fitted GAMLSS residuals pass the HSIC independence test at level $\alpha$.
\end{assumption}
LSNM-UV-Base is the LSNM analogue of CAM-UV~\cite{maeda21a}, Algorithm~1. We adapt Algorithms~1 and~2 of \cite{maeda21a} from additive residuals to location-scale residuals. The algorithm has three phases:  In \emph{Phase~1}, we iterate over subsets $K \subseteq X$ of increasing size and identify \emph{sinks}. A variable $x_i \in K$ is a sink if no other member of $K$ is a descendant of $x_i$, so $K \setminus \{x_i\}$ can be added to $x_i$'s parent candidate set~$M_i$. The sink test (Lemma~\ref{lemma:sink}) checks two conditions detailed in Appendix \ref{base-algorithm-section}, first in Equation \eqref{eq:sink_indep}, after regressing $x_i$ on $M_i \cup K \setminus \{x_i\}$, the residual is independent of each of the other member's residual, meaning $K$'s parental contribution has been absorbed analogous to Equation~\eqref{eq:case2b-main}; and second in Equation \eqref{eq:sink_dep}, without $K$ in the regression (regressing on $M_i$ only), all residual pairs remain dependent, confirming $K$ carried genuine parent information that cannot be removed, analogous to Equation~\eqref{eq:case2a-main}. 



In \emph{Phase~2}, we determine the parents of each variable. For each $x_j \in M_i$, if the residual of $x_i$ regressed on $M_i \setminus \{x_j\}$ is independent of the residual of $x_j$ regressed on $M_j$, then $x_j$ is not a parent of $x_i$ (Proposition~\ref{prop:residual_independence_pattern}, Case~3,~\eqref{eq:case3-main}) and is removed from~$M_i$. If $x_j$ is a true parent, the direct causal effect $x_j \to x_i$ is not blocked by $M_i \setminus \{x_j\}$, so the residuals remain dependent (Case~2,~\eqref{eq:case2a-main}) and $x_j$ is retained. Retained non-parents (UBP/UCP partners) are distinguished in Phase~3.

In \emph{Phase~3}, we set $A(i,j) = 1$ for each $x_j \in M_i$ and mark non-adjacent pairs as invisible when their residuals $h_i(i, X \setminus M_i)$ and $h_j(j, X \setminus M_j)$ are dependent (Proposition~\ref{prop:residual_independence_pattern}, Case~1).

The sink test requires independence with \emph{every} member of $K \setminus \{x_i\}$. If any $x_j \in K$ satisfies $B(i,j) = 1$, the shared hidden noise forces $h_i \notcind h_j$ regardless of regression set (Proposition~\ref{prop:residual_independence_pattern}, Case~1), so the sink test fails for \emph{any} $K$ containing~$x_j$. True parents of $x_i$ that appear only in subsets alongside such $x_j$ are never discovered, leaving $M_i$ incomplete. 


In Stage~2, we apply the \textsc{checkVisible} procedure of \citet{pham2025} unmodified (Algorithm~\ref{alg:checkvisible}), which re-examines each pair $(x_i, x_j)$ marked invisible by Stage~1. It builds a search set $Q$ containing all identified parents and all variables with unresolved relationships to $x_i$ or~$x_j$. It checks if adding $x_j$ to $x_i$'s regression along with their respective parents. If independence holds, then $(x_i, x_j)$ is not invisible and $x_i$ is not a parent of~$x_j$ (Lemma~\ref{lem:nonparent-inclusion}, Equation ~\eqref{eq:lem9-a}).  Symmetric test adding $x_i$ to $x_j$'s regression (Lemma~\ref{lem:nonparent-inclusion}, Equation~\eqref{eq:lem9-b}). If both non-parentship tests pass, the pair is a visible non-edge.



Soundness and completeness are formally defined in Definition~\ref{def:soundness_completeness} (Appendix).
\begin{theorem}\label{thm:soundness_completeness}
With Assumptions~\ref{assumption:CFC},~\ref{assumption:residual_faithfulness}, and~\ref{ass:HSIC}, Algorithm~\ref{alg:main} (LSNM-UV) is sound and complete in identifying visible edges, visible non-edges, and invisible pairs.
\end{theorem}
For full proof, see Section \ref{sound-complete-proof}.
\section{Experiments}
\subsection{Performance on artificial data}

We follow the setup of \citet{maeda21a} Section~5.1: $p = 10$ observed variables with an Erd\H{o}s--R\'enyi DAG (edge probability $0.3$), augmented with $2$ hidden common causes (inducing UBPs) and $2$ hidden intermediates (inducing UCPs); a representative graph is shown in Figure~\ref{fig:datagen}. Data are generated from the two-level LSNM in Equation~\eqref{eq: LSNM-UV} using nonlinear families (RBF, tanh, softplus, logarithmic) for the structural functions $f^1, g^1, f^2, g^2$ respectively; full details are in Appendix~\ref{app:datagen}. We evaluate on sample sizes $n \in \{200, 400, \ldots, 1000\}$ over $400$ independent trials, and compare against \textbf{CAM-UV}~\citep{maeda21a}, \textbf{FCI}~\citep{spirtes2000,zhang2008}, and \textbf{BANG}~\citep{wang2020bang}. Precision, recall, and F1 are reported separately for directed and bidirected edges.
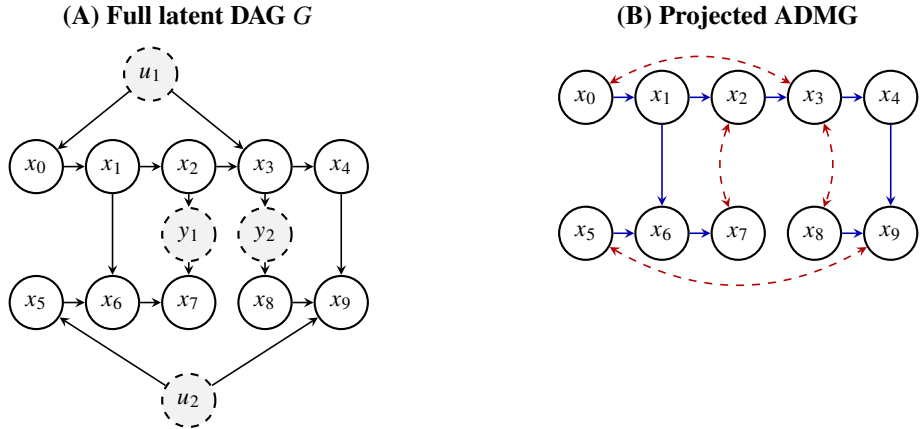
\begin{figure}[H]
\centering
\begin{minipage}[t]{0.48\textwidth}
\centering
\textbf{(A) Full latent DAG $G$}\\[6pt]
\begin{tikzpicture}[
  obs/.style={circle, draw, thick, minimum size=20pt, inner sep=0pt, font=\small},
  hid/.style={circle, draw, dashed, thick, minimum size=20pt, inner sep=0pt,
              font=\small, fill=gray!10},
  arr/.style={->, >=stealth, semithick},
  scale=0.72,
]
\node[obs] (x0) at (0, 2.5)    {$x_0$};
\node[obs] (x1) at (1.4, 2.5)  {$x_1$};
\node[obs] (x2) at (2.8, 2.5)  {$x_2$};
\node[obs] (x3) at (4.2, 2.5)  {$x_3$};
\node[obs] (x4) at (5.6, 2.5)  {$x_4$};
\node[obs] (x5) at (0, 0)      {$x_5$};
\node[obs] (x6) at (1.4, 0)    {$x_6$};
\node[obs] (x7) at (2.8, 0)    {$x_7$};
\node[obs] (x8) at (4.2, 0)    {$x_8$};
\node[obs] (x9) at (5.6, 0)    {$x_9$};
\node[hid] (u1) at (2.1, 4.2)  {$u_1$};
\node[hid] (u2) at (2.8, -1.8) {$u_2$};
\node[hid] (y1) at (2.8, 1.25) {$y_1$};
\node[hid] (y2) at (4.2, 1.25) {$y_2$};

\draw[arr] (x0) -- (x1);
\draw[arr] (x1) -- (x2);
\draw[arr] (x2) -- (x3);
\draw[arr] (x3) -- (x4);
\draw[arr] (x5) -- (x6);
\draw[arr] (x6) -- (x7);
\draw[arr] (x8) -- (x9);
\draw[arr] (x1) -- (x6);
\draw[arr] (x4) -- (x9);

\draw[arr] (u1) -- (x0);
\draw[arr] (u1) -- (x3);

\draw[arr] (u2) -- (x5);
\draw[arr] (u2) -- (x9);

\draw[arr] (x2) -- (y1);
\draw[arr] (y1) -- (x7);

\draw[arr] (x3) -- (y2);
\draw[arr] (y2) -- (x8);
\end{tikzpicture}
\end{minipage}%
\hfill
\begin{minipage}[t]{0.48\textwidth}
\centering
\textbf{(B) Projected ADMG}\\[6pt]
\begin{tikzpicture}[
  obs/.style={circle, draw, thick, minimum size=20pt, inner sep=0pt, font=\small},
  dir/.style={->, >=stealth, semithick, blue!70!black},
  bidir/.style={<->, >=stealth, semithick, red!70!black, dashed},
  scale=0.72,
]
\node[obs] (x0) at (0, 2.5)    {$x_0$};
\node[obs] (x1) at (1.4, 2.5)  {$x_1$};
\node[obs] (x2) at (2.8, 2.5)  {$x_2$};
\node[obs] (x3) at (4.2, 2.5)  {$x_3$};
\node[obs] (x4) at (5.6, 2.5)  {$x_4$};
\node[obs] (x5) at (0, 0)      {$x_5$};
\node[obs] (x6) at (1.4, 0)    {$x_6$};
\node[obs] (x7) at (2.8, 0)    {$x_7$};
\node[obs] (x8) at (4.2, 0)    {$x_8$};
\node[obs] (x9) at (5.6, 0)    {$x_9$};

\draw[dir] (x0) -- (x1);
\draw[dir] (x1) -- (x2);
\draw[dir] (x2) -- (x3);
\draw[dir] (x3) -- (x4);
\draw[dir] (x5) -- (x6);
\draw[dir] (x6) -- (x7);
\draw[dir] (x8) -- (x9);
\draw[dir] (x1) -- (x6);
\draw[dir] (x4) -- (x9);

\draw[bidir, bend left=30] (x0) to (x3);   
\draw[bidir, bend right=30] (x5) to (x9);  
\draw[bidir, bend right=20] (x2) to (x7);  
\draw[bidir, bend left=20] (x3) to (x8);   
\end{tikzpicture}
\end{minipage}

\caption{Simulation graph ($p{=}10$, $2$ UBPs, $2$ UCPs).
\textbf{(A)}~Full latent DAG: solid = observed, dashed = hidden. Common causes $u_1, u_2$ induce UBPs; intermediates $y_1, y_2$ induce UCPs.
\textbf{(B)}~Projected bow-free ADMG: \textcolor{blue!70!black}{blue} = directed, \textcolor{red!70!black}{red dashed} = bidirected.}
\label{fig:datagen}
\end{figure}
\paragraph{Results.}
Precision is the fraction of predicted edges that are correct; recall is the fraction of true edges recovered. Figure~\ref{fig:results} shows the results. LSNM-UV achieves the highest directed F1 across all sample sizes ($0.53$ at $n{=}200$, $0.68$ at $n{=}1000$), with precision $0.66$ and recall $0.71$ at $n{=}1000$. CAM-UV reaches only $0.13$ directed F1 at $n{=}1000$ (precision $0.25$, recall $0.10$), as its additive residuals cannot account for heteroscedasticity; FCI and BANG remain below $0.04$. For bidirected edges, CAM-UV and BANG obtain higher recall ($0.62$ and $0.75$) but low precision ($0.14$), producing many false positives. LSNM-UV recovers fewer bidirected edges (recall $0.07$) but with the highest precision ($0.21$).

\begin{figure}[H]
\centering
\includegraphics[width=0.9\textwidth]{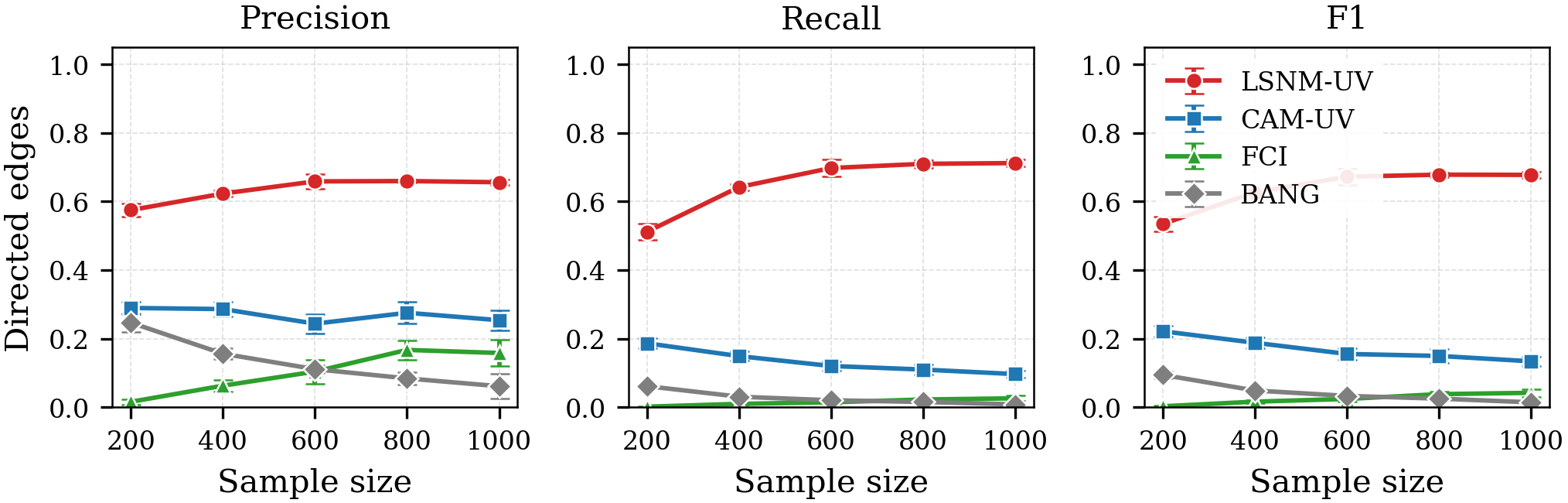}\\[4pt]
\includegraphics[width=0.9\textwidth]{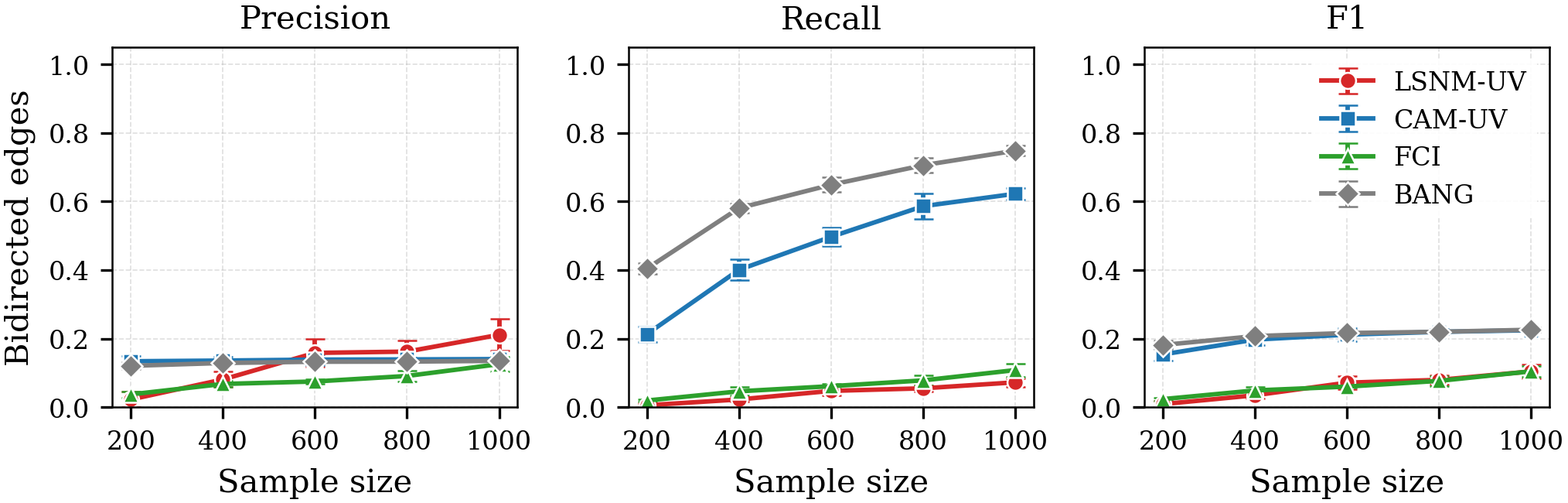}
\caption{Precision, recall, and F1 for directed edges (top row) and bidirected edges (bottom row). Data are generated from the LSNM-UV model (Equation~\eqref{eq: LSNM-UV}) with $f^1{=}\text{RBF}$, $g^1{=}\text{tanh}$, $f^2{=}\text{softplus}$, $g^2{=}\text{logarithmic}$, using the graph in Figure~\ref{fig:datagen}. Each point averages over $400$ independent trials at sample sizes $n \in \{200, 400, \ldots, 1000\}$; error bars show $\pm 1$ standard deviation across four batches of $100$ trials.}
\label{fig:results}
\end{figure}

\section{Conclusion}

We extended causal additive models to the location-scale noise setting with unobserved variables. Our theoretical analysis showed that LSNM residuals characterise directed edges, bidirected edges, and non-edges (Proposition~\ref{prop:residual_independence_pattern}), that bow-free ADMGs are identifiable under this model (Theorem~\ref{thm:ADMG_identifiability}), and provided sufficient conditions for identifying causal direction even in the presence of bows (Lemma~\ref{lem:checkonpath}). The proposed algorithm is provably sound and complete (Theorem~\ref{thm:soundness_completeness}), and experiments confirm that it substantially outperforms additive-noise baselines in directed-edge recovery on heteroscedastic data. A current limitation is low bidirected-edge recall; future work will focus on improving bidirected detection and on applying the method to real-world observational data.
 See Appendix~\ref{appendix:broader_impacts} for discussion on broader impacts.
%
\section*{Acknowledgements}
This work was partially supported by the Japan Science and Technology Agency (JST) under CREST Grant Number JPMJCR22D2 and by the Japan Society for the Promotion of Science (JSPS) under KAKENHI Grant Numbers  JP24K20741 and JP25K03084.




\appendix

\section{Technical appendices and supplementary material}
\subsection{Theoretical proofs}
\paragraph{Notation.} In the main text, the regression function $h(i,S)$ uses a single exclusion set $S$ for both the location function $f$ and the scale function $g$. The proofs below establish all results under the more general formulation where $f$ and $g$ may exclude different subsets $S_f, S_g \subseteq S$. Since the unified case $S_f = S_g = S$ is a special case of the general formulation, all main-text results follow immediately.

\subsubsection{Proof Lemma~\ref{lemma:independence_patterns}}
\begin{proof}[Proof of Lemma~\ref{lemma:independence_patterns}]

We assume mutual independence of exogenous noises $\epsilon_l$ for all $v_l \in V$ in~\eqref{eq: LSNM-UV} and the residual faithfulness condition (Assumption~\ref{assumption:residual_faithfulness}).

Recall from~\eqref{eq: LSNM-dependent-noise} that the effective noise for each observed variable $x_i$ is 
\begin{align}
\label{eq:eta_def}
    \eta_i = f_{i}^2(Q_i) + g_{i}^2(Q_i)\epsilon_i.
\end{align}
where $Q_i$ is the set of unobserved parents of $x_i$. We note that $\eta_i$ can equivalently be written as a location-scale regression residual of observed variables
\begin{align}
    \label{eq:eta_as_residual}
    \eta_i = \frac{x_i - f_i^1(K_i)}{g_i^1(K_i)} = h_1(i,S)
\end{align}
where $K_i$ is the set of observed parents of $x_i$. For $j$, analogously we would have some $\eta_j = h_2(j,S')$.

Since the causal graph $G$ is a DAG, every variable $v_l \in V$ can be written as a deterministic function of the exogenous noises of its ancestors and itself. In particular, each unobserved parent $u_k \in Q_i$ is a deterministic function of $\{\epsilon_l : v_l \in \mathrm{an}(u_k) \cup \{u_k\} \}$. Therefore, $\eta_i$ in \eqref{eq:eta_def} is a deterministic function of the noise set
\[
N_i = \{ \epsilon_l: v_l \in \mathrm{an}(Q_i) \cup Q_i \cup\{x_i\}\}
\]
i.e.  $\eta_i = \phi_i(N_i)$ for some deterministic function $\phi_i$. We define $N_j$ analogously for $\eta_j$.

\paragraph{No UBP/UCP $\Longrightarrow \eta_i \cind \eta_j$.}
We first prove that if there is no UBP and no UCP between $x_i$ and $x_j$, then $N_i \cap N_j = \emptyset$ which implies that $\eta_i$ and $\eta_j$ are independent. 

Let's say there does exist some $\epsilon_l \in N_i \cap N_j$, for some $l$. We consider three cases depending on whether $v_l$ is one of $x_i, x_j$, a hidden variable or another observed variable.

\medskip
\noindent\emph{Case 1: $v_l = x_i$ (or symmetrically $v_l = x_j$)} Then $\epsilon_i \in N_j$, which means $x_i \in \mathrm{an}(Q_j) \cup Q_j$. Since $x_i$ is observed and $Q_j$ consists of hidden variables, we have $x_i \notin Q_j$, so $x_i \in \mathrm{an}(Q_j)$. This means there exists a hidden variable $u_k \in Q_j$ such that $x_i$ is an ancestor of $u_k$. Tracing the path: $x_i \to \cdots \to u_k \to x_j$, where $u_k \notin X$. This is a UCP from $x_i$ to $x_j$, contradicting our assumption. The case $v_l = x_j$ follows symmetrically, producing a UCP from $x_j$ to $x_i$.

\medskip
\noindent\emph{Case 2: $v_l \in U$ (an unobserved variable)} Then $v_l \in \mathrm{an}(Q_i) \cup Q_i$ and $v_l \in \mathrm{an}(Q_j) \cup Q_j$. Hence there exist directed paths from $v_l$ to some $u_{k_1} \in Q_i$ (a hidden parent of $x_i$) and from $v_l$ to some $u_{k_2} \in Q_j$ (a hidden parent of $x_j$). Together these give a path of the form $x_i \leftarrow u_{k_1} \leftarrow \cdots \leftarrow v_l \to \cdots \to u_{k_2} \to x_j$ where $u_{k_1}, u_{k_2} \notin X$. This is a UBP between $x_i$ and $x_j$, contradicting our assumption.

\medskip
\noindent\emph{Case 3: $v_l \in X$ with $v_l \neq x_i$ and $v_l \neq x_j$ (an observed variable)} Then $v_l$ is an observed variable that is an ancestor of hidden parents of both $x_i$ and $x_j$. That is, there exists $u_{k_1} \in Q_i$ and $u_{k_2} \in Q_j$ such that $v_l \in \mathrm{an}(u_{k_1})$ and $v_l \in \mathrm{an}(u_{k_2})$. The path $x_i \leftarrow u_{k_1} \leftarrow \cdots \leftarrow v_l \to \cdots \to u_{k_2} \to x_j$ constitutes a UBP between $x_i$ and $x_j$ (the source node $v_l$ may be observed, but $u_{k_1}, u_{k_2} \notin X$), contradicting our assumption.

In all cases, the existence of a shared noise term $\epsilon_l \in N_i \cap N_j$ implies the existence of a UCP or UBP, which contradicts the hypothesis. Therefore, $N_i \cap N_j = \emptyset$.

Since the exogenous noise terms $\epsilon_l$ are mutually independent, the sets $N_i$ and $N_j$ are collections of independent random variables with no overlap. Therefore $\eta_i = \phi_i(N_i)$ and $\eta_j = \phi_j(N_j)$ are functions of independent noise terms. This implies that 

\[
\eta_i \cind \eta_j
\]
\paragraph{$\eta_i \cind \eta_j \Longrightarrow $ No UBP/UCP.}
We prove that if $\eta_i$ and $\eta_j$ are independent then there is no UBP and no UCP between $x_i$ and $x_j$.

\medskip
\noindent\emph{Case A: A UBP exists between $x_i$ and $x_j$.} By definition, there is a path of the form $x_i \leftarrow u_{k_1} \leftarrow \cdots \leftarrow v_a \to \cdots \to u_{k_2} \to x_j$ where $u_{k_1} \in Q_i$ and $u_{k_2} \in Q_j$ are hidden variables. Then there is a common ancestor $v_a$ to both $x_i$ and $x_j$ with exogenous noise $\epsilon_a$. 

Since $v_a \in \mathrm{an}(u_{k_1})$ and $u_{k_1} \in Q_i$, we have that $\epsilon_a \in N_i$.

Since $v_a \in \mathrm{an}(u_{k_2})$ and $u_{k_2} \in Q_j$, we have that $\epsilon_a \in N_j$.

Thus $\epsilon_a \in N_i \cap N_j$ and both $\eta_i$ and $\eta_j$ are functions of $\epsilon_a$. This implies that both $h_1(i,S)$ and $h_2(j,S')$ depend on the same noise term $\epsilon_a$. Thus by the \emph{residual faithfulness condition} (Assumption~\ref{assumption:residual_faithfulness}) gives $ \eta_i \notcind \eta_j$.

\medskip
\noindent\emph{Case B: A UCP exists from $x_i$ to $x_j$ (or vice versa).} 
If a UCP exists between $x_i$ to $x_j$, then by definition, we would have a path $x_i \to \cdots \to u_k \to x_j$ where $u_k \in Q_j$ and $u_k \notin X$. Since $x_i$ is an ancestor of $u_k$, the variable $u_k$ depends on $\epsilon_i$ through the directed path $x_i \to \cdots \to u_k$. Since $u_k \in Q_j$, we have $\epsilon_i \in N_j$. 

Meanwhile $\epsilon_i \in N_i$ directly (since $x_i \in \{x_i\} \subseteq \{x_i\}  \cup Q_i \cup \mathrm{an}(Q_i)$). 

Thus $\epsilon_i \in N_i \cap N_j$ and both $\eta_i$ and $\eta_j$ are functions of $\epsilon_i$. This implies that both $h_1(i,S)$ and $h_2(j,S')$ depend on the same noise term $\epsilon_i$. By the \emph{residual faithfulness condition} (Assumption~\ref{assumption:residual_faithfulness}) gives $ \eta_i \notcind \eta_j$.

The case where a UCP exists from  $x_j$ and $x_i$ follows by symmetry. 

In all cases, the existence of a UCP or UBP implies that $ \eta_i \notcind \eta_j$, hence the implication that  $ \eta_i \cind \eta_j$ implies No UCP/UBP.
\end{proof}

\subsubsection{Proof of Proposition \ref{prop:residual_independence_pattern}}
For proof of Proposition \ref{prop:residual_independence_pattern},  we will use two results. First is from \cite{maeda21a}, 
\begin{lemma}[Non-decomposability of nonlinear compositions]
\label{lemma:non_decomp}
Let $\varphi$ be a nonlinear function, and let $a$ and $b$ be two (independent) random variables. Then $\varphi(a + b)$ cannot be decomposed as $s(a) + t(b)$ for any functions $s, t$. Hence for any non-linear function $\varphi(a + b)$, it cannot be represented as a linear sum of functions of $a$ and $b$ \cite{maeda21a}. 
\end{lemma}
Second result is an extension of Lemma~4 in \cite{ashman2023causal}, for our LSNM-UV:
\begin{lemma}[Irremovability of own noise]
\label{lemma:noise-irremovability-LSNM-residuals}
Let $s(x_i)$ be an arbitrary function of $x_i$. Under the LSNM-UV model in \eqref{eq: LSNM-UV}, for all location-scale regression functions $h \in \mathcal{H}$:
\[
\forall\, f \in \mathcal{F},\; g \in \mathcal{F}_{>0}: \quad \frac{s(x_i) - f(\mathbf{x}_{-i})}{g(\mathbf{x}_{-i})} \;\notcind\; \epsilon_i,
\]
where $\mathbf{x}_{-i} = X \setminus \{x_i\}$.
\end{lemma}

\begin{proof}[Proof of Lemma \ref{lemma:noise-irremovability-LSNM-residuals}]
Suppose for contradiction that there exist $f \in \mathcal{F}$ and $g \in \mathcal{F}_{>0}$ such that
\begin{equation}
\label{eq:suppose_indep}
r_i := \frac{s(x_i) - f(\mathbf{x}_{-i})}{g(\mathbf{x}_{-i})} \;\cind\; \epsilon_i.
\end{equation}
Since the DAG $G$ is acyclic, every variable $v_l \in V$ is a deterministic function of the exogenous noises $\{\epsilon_l\}$. Therefore $r_i$ is a deterministic function of $(\boldsymbol{\epsilon}_{-i}, \epsilon_i)$, where $\boldsymbol{\epsilon}_{-i}$ denotes all exogenous noises except $\epsilon_i$. If $r_i \cind \epsilon_i$, then $r_i$ cannot vary with $\epsilon_i$, so there exists a function $q$ such that
\begin{equation}
\label{eq:ri_is_q}
r_i = q(\boldsymbol{\epsilon}_{-i}).
\end{equation}
Rearranging~\eqref{eq:suppose_indep} and~\eqref{eq:ri_is_q}:
\begin{equation}
\label{eq:rearranged}
s(x_i) = f(\mathbf{x}_{-i}) + g(\mathbf{x}_{-i}) \cdot q(\boldsymbol{\epsilon}_{-i}).
\end{equation}
Call the right side $R(\boldsymbol{\epsilon}_{-i}, \epsilon_i)$.

Since $x_i$ depends on $\epsilon_i$ through \eqref{eq: LSNM-UV}, $s(x_i)$ also depends on $\epsilon_i$. Crucially, $s(x_i)$ does \emph{not} depend on the noise $\epsilon_k$ of any descendant $x_k$ of $x_i$, because $\epsilon_k$ is generated at $x_k$ which comes \emph{after} $x_i$ in the causal order.

Consider the right side $R(\boldsymbol{\epsilon}_{-i}, \epsilon_i) = f(\mathbf{x}_{-i}) + g(\mathbf{x}_{-i}) \cdot q(\boldsymbol{\epsilon}_{-i})$. The variables $\mathbf{x}_{-i}$ can only access $\epsilon_i$ through \emph{descendants} of $x_i$ in $\mathbf{x}_{-i}$. Any descendant $x_k$ of $x_i$ has a structural equation of the form 
\[
x_k = f_k^1(K_k) + g_k^1(K_k) \cdot \eta_k,
\]
where $x_i \in K_k$ or $x_i$ is an ancestor of some variable in $K_k$ and $\eta_k$ involves $\epsilon_k$ (the exogenous noise \emph{at} $x_k$). The dependence of $x_k$ on $\epsilon_i$ arises because $x_i$ (or a descendant of $x_i$) appears in $K_k$, entering through the nonlinear functions $f_k^1$ and $g_k^1$. At the same time, $\epsilon_k$ enters $x_k$ through $\eta_k$.

By the LSNM structural equations, $x_k$ depends on $\epsilon_i$ (inherited through the chain of ancestors) and on $\epsilon_k$ (its own exogenous noise). By Lemma~\ref{lemma:non_decomp}, since $f_k^1$ and $g_k^1$ are nonlinear, there do not exist functions $p$ and $r$ such that
\[
x_k = p(\epsilon_i, \boldsymbol{\epsilon}_{\text{other}}) + r(\epsilon_k, \boldsymbol{\epsilon}_{\text{other}}).
\]
Consequently, any function $\phi(x_k, \ldots)$ that depends on $\epsilon_i$ through $x_k$ also depends on $\epsilon_k$ through $x_k$.

But as we saw before, the left side $s(x_i)$ does \emph{not} depend on $\epsilon_k$ (the noise at a descendant $x_k$). The equation $s(x_i) = R(\boldsymbol{\epsilon}_{-i}, \epsilon_i)$ forces the right side to not depend on $\epsilon_k$, but this is impossible.
\end{proof}
\begin{proof}[Proof of Proposition \ref{prop:residual_independence_pattern}]
Assume the data is generated from the SCM in \eqref{eq: LSNM-UV} 
with true graph $G$. Let $T_{ADMG}(G) = (A, B)$. Assume the residual 
faithfulness condition in Assumption~\ref{assumption:residual_faithfulness}. 
Then for each pair of observed variables $x_i, x_j$:

\noindent\textbf{Case 1 (Bidirected edge):}
\begin{equation}
\label{eq:case1}
B(i,j) = 1 \quad \iff \quad \forall\, h_1, h_2 \in \mathcal{H}: 
\; h_1(i, \{i\}) \notcind h_2(j, \{j\}).
\end{equation}
Recall from~\eqref{eq:eta_def} that $\eta_i = f_{i}^2(Q_i) + g_{i}^2(Q_i)\epsilon_i$, and from~\eqref{eq:eta_as_residual} that $\eta_i = (x_i - f_i^1(K_i))/g_i^1(K_i)$. For the choice $S_f = S_g = \{i\}$, the true structural functions $f_i^1, g_i^1$ use $K_i \subseteq X \setminus \{x_i\} = X \setminus S_f$, so the constraint on $\mathcal{H}$ is satisfied and $\eta_i = h_1(i, \{i\})$ when the true functions are chosen.

$B(i,j) = 1$ means there is a UBP or UCP between $x_i$ and $x_j$. 
If there is a UBP between $x_i$ and $x_j$ then $x_i$ has an unobserved 
direct cause $u_m \in Q_i$ and $x_j$ has an unobserved direct cause 
$u_n \in Q_j$, and both are connected through a common ancestor $v_l$ 
with noise $\epsilon_l$. Then $\epsilon_l$ propagates through $u_m$ 
into $\eta_i$ and through $u_n$ into $\eta_j$ through non-linear 
transformations of hidden variables.

The residual $h_1(i, \{i\})$ is a location-scale regression of $x_i$ 
on observed variables in $X \setminus \{x_i\}$, and $h_2(j, \{j\})$ 
is a location-scale regression of $x_j$ on observed variables in 
$X \setminus \{x_j\}$. Since $\epsilon_l$ enters $x_i$ through hidden 
variables in $Q_i$, and the regression functions in $h_1$ use only 
observed variables, the contribution of $\epsilon_l$ to $x_i$ cannot 
be cancelled by any regression on $X \setminus \{x_i\}$, by 
Lemma~\ref{lemma:non_decomp}. Therefore 
$\epsilon_l \notcind h_1(i, \{i\})$ for all $h_1 \in \mathcal{H}$. 
By the same argument, $\epsilon_l \notcind h_2(j, \{j\})$ for all 
$h_2 \in \mathcal{H}$. By the residual faithfulness condition 
(Assumption~\ref{assumption:residual_faithfulness}), 
$h_1(i,\{i\}) \notcind h_2(j,\{j\})$ for all $h_1, h_2 \in \mathcal{H}$. 
The case of a UCP follows by the same argument, with $\epsilon_l$ 
replaced by the appropriate shared noise (see the proof of 
Lemma~\ref{lemma:independence_patterns}, Cases A and B).

For the other direction, suppose $B(i,j) = 0$. Choose $h_1$ with 
the true functions for $x_i$: since $x_i$ cannot be its own parent, 
$x_i \notin K_i$, so $S_f = S_g = \{i\}$ satisfies the constraint 
(as $K_i \subseteq X \setminus \{x_i\}$), giving 
$h_1(i,\{i\}) = \eta_i$. Similarly $h_2(j,\{j\}) = \eta_j$. By 
Lemma~\ref{lemma:independence_patterns}, no UBP/UCP implies 
$\eta_i \cind \eta_j$. This gives $\exists\, h_1, h_2$ achieving 
independence, contradicting $\forall$ dependence.

\noindent\textbf{Case 2 (Directed edge $x_j \to x_i$, no confounding):}
\begin{align}
A(i,j) = 1 \;\wedge\; B(i,j) = 0 \quad \iff \quad 
& \forall\, h_1, h_2 \in \mathcal{H}: \; 
  h_1(i, \{i,j\}) \notcind h_2(j, \{j\}), \quad \text{and} 
  \label{eq:case2a} \\
& \exists\, h_1, h_2 \in \mathcal{H}: \; 
  h_1(i, \{i\}) \cind h_2(j, \{j\}). \label{eq:case2b}
\end{align}

($\Rightarrow$) Suppose $A(i,j) = 1$ and $B(i,j) = 0$. We show 
both conditions hold.

\emph{Condition~\eqref{eq:case2a}:} Consider any
$h_1 \in \mathcal{H}$ with $S = \{i,j\}$. The constraint $S = \{i,j\}$ 
requires at least one of $S_f = \{i,j\}$ or $S_g = \{i,j\}$, so 
$x_j$ is excluded from the argument of at least one of $f$ and $g$. 
Since $A(i,j) = 1$, $x_j$ is a parent of $x_i$, and both $f_i^1$ 
and $g_i^1$ depend on $x_j$. Because $x_j$ is excluded from at 
least one of $f$ and $g$, that function cannot cancel $x_j$'s 
contribution. By Lemma~\ref{lemma:non_decomp} (see Proposition~\ref{prop:regression_sets_fg}(a,b) below for the full case analysis covering both $f$-only and $g$-only exclusion), the residual 
$h_1(i, \{i,j\})$ retains dependence on $\epsilon_j$, i.e., 
$\epsilon_j \notcind h_1(i, \{i,j\})$.

Meanwhile, consider any $h_2 \in \mathcal{H}$ with $S = \{j\}$. Since
$\epsilon_j$ is the exogenous noise of $x_j$ and enters $\eta_j$
irremovably, $h_2(j, \{j\})$ retains $\epsilon_j$ by 
Lemma~\ref{lemma:noise-irremovability-LSNM-residuals}, i.e., 
$\epsilon_j \notcind h_2(j, \{j\})$.

Since the choice of $h_1$ and $h_2$ was arbitrary, both share 
$\epsilon_j$ for all valid $h_1, h_2 \in \mathcal{H}$. By the 
residual faithfulness condition 
(Assumption~\ref{assumption:residual_faithfulness}), 
$h_1(i, \{i,j\}) \notcind h_2(j, \{j\})$ for all 
$h_1, h_2 \in \mathcal{H}$.

\emph{Condition~\eqref{eq:case2b}:} Choose $h_1$ using the true 
structural functions for $x_i$: $f = f_i^1$ and $g = g_i^1$. 
Set $S_f = S_g = \{i\}$. Since $x_i \notin K_i$ (DAG property), 
$K_i \subseteq X \setminus \{x_i\}$, so both $f_i^1$ and $g_i^1$ 
use only variables in $X \setminus \{x_i\} = X \setminus S_f$, 
satisfying the constraint for $S = \{i\}$. Since 
$x_j \in K_i$ (as $x_j$ is a parent of $x_i$), $x_j$ is included 
in the arguments of both $f_i^1$ and $g_i^1$. This gives:
\[
h_1(i, \{i\}) = \frac{x_i - f_i^1(K_i)}{g_i^1(K_i)} = \eta_i.
\]

For $h_2$, choose the true structural functions for $x_j$: 
$f = f_j^1$ and $g = g_j^1$. Set $S_f = S_g = \{j\}$. Since 
$x_j \notin K_j$ (DAG property), $K_j \subseteq X \setminus \{x_j\}$, 
so both $f_j^1$ and $g_j^1$ use only variables in 
$X \setminus \{x_j\} = X \setminus S_f$, satisfying the constraint 
for $S = \{j\}$. This gives:
\[
h_2(j, \{j\}) = \frac{x_j - f_j^1(K_j)}{g_j^1(K_j)} = \eta_j.
\]

Since $B(i,j) = 0$ (no UBP/UCP), 
Lemma~\ref{lemma:independence_patterns} gives $\eta_i \cind \eta_j$. 
Thus $\exists\, h_1, h_2 \in \mathcal{H}$ achieving independence.

($\Leftarrow$) Suppose both conditions~\eqref{eq:case2a} 
and~\eqref{eq:case2b} hold. We show $A(i,j) = 1$ and $B(i,j) = 0$.

\emph{Showing $B(i,j) = 0$:} Condition~\eqref{eq:case2b} gives 
$\exists\, h_1, h_2 \in \mathcal{H}$: 
$h_1(i, \{i\}) \cind h_2(j, \{j\})$. This is the negation of the 
universal dependence in Case~1 (which states 
$\forall\, h_1, h_2 \in \mathcal{H}$: 
$h_1(i, \{i\}) \notcind h_2(j, \{j\})$), so $B(i,j) = 0$.

\emph{Showing $A(i,j) = 1$:} Suppose for contradiction that 
$A(i,j) = 0$, i.e., $x_j$ is not a parent of $x_i$. Then 
$x_j \notin K_i$, so $K_i \subseteq X \setminus \{x_i, x_j\}$. 
Choose $h_1$ with the true structural functions and 
$S_f = S_g = \{i,j\}$. Since 
$K_i \subseteq X \setminus \{x_i, x_j\} = X \setminus S_f$, 
both $f_i^1$ and $g_i^1$ use only variables in $X \setminus S_f$, 
satisfying the constraint for $S = \{i,j\}$. This gives 
$h_1(i, \{i,j\}) = \eta_i$. 

For $h_2(j, \{j\})$, the true structural functions for $x_j$ give 
$h_2(j, \{j\}) = \eta_j$ (as shown in the forward direction). 

We have already established $B(i,j) = 0$, so by 
Lemma~\ref{lemma:independence_patterns}, $\eta_i \cind \eta_j$. But 
this contradicts Condition~\eqref{eq:case2a}, which asserts 
$h_1(i, \{i,j\}) \notcind h_2(j, \{j\})$ for all 
$h_1, h_2 \in \mathcal{H}$. Therefore $A(i,j) = 1$.

\noindent\textbf{Case 3 (No edge):}
\begin{equation}
\label{eq:case3}
A(i,j) = 0 \;\wedge\; A(j,i) = 0 \;\wedge\; B(i,j) = 0 
\quad \iff \quad 
\exists\, h_1, h_2 \in \mathcal{H}: \; 
h_1(i, \{i,j\}) \cind h_2(j, \{i,j\}).
\end{equation}

($\Rightarrow$) Suppose $A(i,j) = 0$, $A(j,i) = 0$, and 
$B(i,j) = 0$. We show $\exists\, h_1, h_2 \in \mathcal{H}$ such 
that $h_1(i, \{i,j\}) \cind h_2(j, \{i,j\})$.

Since $A(i,j) = 0$, $x_j$ is not a parent of $x_i$, and since 
$x_i$ cannot be its own parent in a DAG, we have 
$K_i \subseteq X \setminus \{x_i, x_j\}$. Choose $h_1$ using the 
true structural functions $f = f_i^1$ and $g = g_i^1$, with 
$S_f = S_g = \{i,j\}$. Since 
$K_i \subseteq X \setminus \{x_i, x_j\} = X \setminus S_f$, 
both $f_i^1$ and $g_i^1$ use only variables in $X \setminus S_f$, 
satisfying the constraint for $S = \{i,j\}$. This gives:
\[
h_1(i, \{i,j\}) = \frac{x_i - f_i^1(K_i)}{g_i^1(K_i)} = \eta_i.
\]

Similarly, since $A(j,i) = 0$, $x_i$ is not a parent of $x_j$, so 
$K_j \subseteq X \setminus \{x_i, x_j\}$. Choosing the true 
structural functions for $x_j$ with $S_f = S_g = \{i,j\}$:
\[
h_2(j, \{i,j\}) = \frac{x_j - f_j^1(K_j)}{g_j^1(K_j)} = \eta_j.
\]

Since $B(i,j) = 0$ (no UBP/UCP), 
Lemma~\ref{lemma:independence_patterns} gives $\eta_i \cind \eta_j$. 
Thus $\exists\, h_1, h_2 \in \mathcal{H}$ achieving independence.

($\Leftarrow$) Suppose $\exists\, h_1, h_2 \in \mathcal{H}$ such 
that $h_1(i, \{i,j\}) \cind h_2(j, \{i,j\})$. We show 
$B(i,j) = 0$, $A(i,j) = 0$, and $A(j,i) = 0$.

\emph{Showing $B(i,j) = 0$:} Suppose for contradiction that 
$B(i,j) = 1$. Then there exists a UBP or UCP between $x_i$ and 
$x_j$, which means there exists a shared exogenous noise 
$\epsilon_l \in N_i \cap N_j$ (as established in the proof of 
Lemma~\ref{lemma:independence_patterns}). Now consider any
$h_1 \in \mathcal{H}$ with $S = \{i,j\}$. The residual $h_1(i, \{i,j\})$ is  computed using regression functions $f$ and $g$ that use only \emph{observed} 
variables. However, $\epsilon_l$ enters $x_i$ through \emph{hidden} 
variables in $Q_i$. Since these hidden variables are not available 
to the regression, the contribution of $\epsilon_l$ to $x_i$ cannot 
be cancelled by any choice of $f$ and $g$ on observed variables, by 
Lemma~\ref{lemma:non_decomp}. Therefore 
$\epsilon_l \notcind h_1(i, \{i,j\})$ for all 
$h_1 \in \mathcal{H}$. By the same argument, 
$\epsilon_l \notcind h_2(j, \{i,j\})$ for all 
$h_2 \in \mathcal{H}$. By the residual faithfulness condition 
(Assumption~\ref{assumption:residual_faithfulness}), 
$h_1(i, \{i,j\}) \notcind h_2(j, \{i,j\})$ for all 
$h_1, h_2 \in \mathcal{H}$, contradicting the assumed existence 
of independent $h_1, h_2$.

\emph{Showing $A(i,j) = 0$:} Suppose for contradiction that 
$A(i,j) = 1$, i.e., $x_j$ is a parent of $x_i$. Consider any
$h_1 \in \mathcal{H}$ with $S = \{i,j\}$. The constraint 
$S = \{i,j\}$ requires at least one of $S_f = \{i,j\}$ or 
$S_g = \{i,j\}$, so $x_j$ is excluded from the argument of at 
least one of $f$ and $g$. Since $x_j$ is a parent of $x_i$ and 
both $f_i^1$ and $g_i^1$ depend on $x_j$, the exclusion of $x_j$ 
from at least one function means the residual $h_1(i, \{i,j\})$ 
retains dependence on $\epsilon_j$ by 
Lemma~\ref{lemma:non_decomp}, i.e., 
$\epsilon_j \notcind h_1(i, \{i,j\})$. Meanwhile, 
$h_2(j, \{i,j\})$ also retains $\epsilon_j$, since $\epsilon_j$ 
is the exogenous noise of $x_j$ and enters $\eta_j$ irremovably 
by Lemma~\ref{lemma:noise-irremovability-LSNM-residuals}, i.e., 
$\epsilon_j \notcind h_2(j, \{i,j\})$. Since the choice of $h_1$ 
and $h_2$ was arbitrary, both share $\epsilon_j$ for all valid 
$h_1, h_2 \in \mathcal{H}$. By the residual faithfulness condition 
(Assumption~\ref{assumption:residual_faithfulness}), 
$h_1(i, \{i,j\}) \notcind h_2(j, \{i,j\})$ for all 
$h_1, h_2 \in \mathcal{H}$, contradicting the assumed existence 
of independent $h_1, h_2$.

\emph{Showing $A(j,i) = 0$:} The argument is symmetric to the 
above, exchanging the roles of $x_i$ and $x_j$: if $x_i$ were a 
parent of $x_j$, then every valid $h_2(j, \{i,j\})$ would retain 
$\epsilon_i$ (since $x_i$ is excluded from at least one of $f$ 
and $g$ by the constraint $S = \{i,j\}$, and $x_i$ is a parent of 
$x_j$, so Lemma~\ref{lemma:non_decomp} gives 
$\epsilon_i \notcind h_2(j, \{i,j\})$). Meanwhile, every valid 
$h_1(i, \{i,j\})$ retains $\epsilon_i$ through $\eta_i$ by 
Lemma~\ref{lemma:noise-irremovability-LSNM-residuals}. 
Residual faithfulness gives dependence, again a contradiction.
\end{proof}
\subsubsection{Proof of Theorem \ref{thm:ADMG_identifiability}}
\begin{proof}[Proof of Theorem \ref{thm:ADMG_identifiability}]
Define the oracle $W_{RI}$ with query $q = (i,j)$ as follows.

When there is a UBP or UCP between $x_i$ and $x_j$, Proposition~\ref{prop:residual_independence_pattern}, Case~1, gives that \eqref{eq:case1-main} holds; the oracle returns $B(i,j) = 1$.

When there is no UBP or UCP between $x_i$ and $x_j$, Case~1 gives that \eqref{eq:case1-main} does not hold; the oracle returns $B(i,j) = 0$. It additionally returns the value of $A(i,j)$: if $x_j \in K_i$, Case~2 gives that \eqref{eq:case2a-main} and \eqref{eq:case2b-main} hold, so the oracle returns $A(i,j) = 1$; otherwise Case~3 gives that \eqref{eq:case3-main} holds, so the oracle returns $A(i,j) = 0$.

These three cases are mutually exclusive. For a bow-free ADMG they are also exhaustive: the bow-free condition $A(i,j)\,B(i,j)=0$ rules out $A(i,j)=1$ and $B(i,j)=1$ simultaneously, so every pair $(x_i,x_j)$ falls into exactly one case. Since each case is determined by dependence and independence tests on the observed distribution, the oracle $W_{RI}$ is a function of the observed distribution, and together mutual exclusivity and exhaustiveness give identifiability of the full ADMG.

\end{proof}

\subsection{Algorithm}\label{sec:algorithm}

We have from Theorem~\ref{thm:ADMG_identifiability} that $T_{ADMG}$ is identifiable if $T_{ADMG}$ is bow-free. When $B(i,j) = 1$, our oracle cannot reveal whether $A(i,j) = 1$ or $A(i,j) = 0$. Bow-free forces $A(i,j) = 0$ in this case. Our claim is that a combinatorial algorithm can correctly classify every pair $(x_i, x_j)$ as one of: ``visible edge,'' ``visible non-edge,'' or ``invisible pair,'' without assuming bow-free. We do not fully recover the ADMG, but we can correctly identify all visible edges (with direction), all visible non-edges, and all invisible pairs. This holds for any DAG.

For full ADMG recovery we need the assumption of bow-free. If the ADMG is bow-free, then invisible pairs correspond exactly to bidirected edges, and combined with visible edges and visible non-edges we recover the full ADMG.

\begin{remark}[Comparison with CAM-UV-X of \cite{pham2025}]
\label{rem:no_subset_visibility}
The CAM-UV-X algorithm of \cite{pham2025} includes three additional stages beyond the base algorithm and \textsc{checkVisible}: \textsc{checkOnPath} (which tests invisibility with respect to subsets $X \setminus \{x_k\}$), \textsc{checkCI} (which uses conditional independence among original variables), and \textsc{checkParentInvi} (which propagates parentship information). These stages rely crucially on the residual independence characterization (Lemmas~1--3 of \cite{maeda21a}) holding for arbitrary subsets $X' \subseteq X$.

In the additive model of \cite{pham2025}, removing an observed parent $x_k$ from the observed set simply adds $x_k$'s contribution to the noise term additively, preserving the additive model structure on $X' = X \setminus \{x_k\}$. Therefore, Lemmas~1--3 can be restated and reproved for $X'$ using the same arguments.

In the LSNM-UV model, removing an observed parent $x_k$ from the observed set does \emph{not} preserve the LSNM structure. If $x_k$ appears in the scale function $g_i^1$ of some $x_i$, the resulting model for $x_i$ on $X' = X \setminus \{x_k\}$ has a noise structure that depends on the hidden $x_k$, which is no longer an LSNM. Consequently, Proposition~\ref{prop:residual_independence_pattern} cannot be applied on subsets $X' \subsetneq X$, and the \emph{language} of subset visibility used in \cite{pham2025} does not transfer.

However, the \emph{core argument} behind \textsc{checkOnPath}, identifying which observed variables lie on invisible paths , does not require subset visibility. It can be stated directly in terms of the GAMLSS residuals $h(i, S)$ on the full data $X$, using the same structural irremovability and faithfulness arguments that underlie Proposition~\ref{prop:residual_independence_pattern}. We formalize this in Lemma~\ref{lem:checkonpath}, which provides the LSNM analogue of Lemma~4 of \cite{pham2025} without invoking subset visibility.
\end{remark}

Let us start by defining visibility/invisibility for our LSNM:
\begin{definition}[Visible parent]
\label{def:vis_parent}
$x_j$ is a \emph{visible parent} of $x_i$ if and only if $x_j \in K_i$ and there is no UBP/UCP between $x_i$ and $x_j$. Equivalently (by Proposition~\ref{prop:residual_independence_pattern}, Case~2):
\begin{align}
\forall\, h_1, h_2 \in \mathcal{H}: \; h_1(i, \{i,j\}) &\notcind h_2(j, \{j\}), \label{eq:vis_parent_dep} \\
\exists\, h_1, h_2 \in \mathcal{H}: \; h_1(i, \{i\}) &\cind h_2(j, \{j\}). \label{eq:vis_parent_indep}
\end{align}
\end{definition}

\begin{definition}[Visible non-edge]
\label{def:vis_nonedge}
$(x_i, x_j)$ is a \emph{visible non-edge} if and only if $A(i,j) = 0$, $A(j,i) = 0$, and there is no UBP/UCP between $x_i$ and $x_j$. Equivalently (by Proposition~\ref{prop:residual_independence_pattern}, Case~3):
\begin{equation}
\exists\, h_1, h_2 \in \mathcal{H}: \; h_1(i, \{i,j\}) \cind h_2(j, \{i,j\}). \label{eq:vis_nonedge}
\end{equation}
\end{definition}

\begin{definition}[Invisible pair]
\label{def:invisible}
$(x_i, x_j)$ is an \emph{invisible pair} if and only if there is a UBP/UCP between $x_i$ and $x_j$. Equivalently (by Proposition~\ref{prop:residual_independence_pattern}, Case~1):
\begin{equation}
\forall\, h_1, h_2 \in \mathcal{H}: \; h_1(i, \{i\}) \notcind h_2(j, \{j\}). \label{eq:invisible}
\end{equation}
\end{definition}

\subsubsection{Regression set requirements for visible parents}
\begin{proposition}[Both $f$ and $g$ require parents]\label{prop:regression_sets_fg}
Let $x_j \in K_i$ (i.e., $x_j$ is a parent of $x_i$). Consider $h(i, S)$ in 
$\mathcal{H}$ with $j \notin S$ (so $x_j$ is available to both $f$ and $g$). 
Then $h(i, S) = \eta_i$ if and only if both $f$ and $g$ use $x_j$ with the 
true structural functions. Specifically:
\begin{enumerate}
\item[(a)] If $f$ uses $x_j$ but $g$ does not (or vice versa), then 
$h(i, S) \neq \eta_i$, and for all choices of $f, g$ under this restriction, 
$h(i, S) \notcind \epsilon_j$.
\item[(b)] If neither $f$ nor $g$ uses $x_j$, then $h(i, S) \neq \eta_i$, 
and for all choices of $f, g$ under this restriction, 
$h(i, S) \notcind \epsilon_j$.
\item[(c)] If both $f$ and $g$ use $x_j$ with the true functions 
$f_i^1, g_i^1$, then $h(i, S) = \eta_i$.
\end{enumerate}
\end{proposition}

\begin{proof}[Proof of Proposition \ref{prop:regression_sets_fg}]
\textbf{Part (c):} We have that $K_i$ is the set of all parents of $x_i$ 
and $x_j \in K_i$, then if $K_i \subseteq X \setminus S$, we will have all 
parents including $x_j$ available for regression. We choose $f = f_i^1$ and 
$g = g_i^1$:
\[
h(i,S) = \frac{x_i- f_i^1(K_i)}{g_i^1(K_i)} 
= \frac{f_{i}^1(K_i) + g_{i}^1(K_i)\eta_i- f_i^1(K_i)}{g_i^1(K_i)} = \eta_i
\]
Since $\eta_i = f_i^2(Q_i) + g_i^2(Q_i)\,\epsilon_i$ does not involve 
$\epsilon_j$ (when $B(i,j) = 0$), we have $h(i,S) \cind \epsilon_j$.

\noindent
\textbf{Part (a):} Suppose $f$ uses $x_j$ but $g$ does not:
\[
h(i,S) = \frac{x_i- f(\ldots,x_j, \ldots)}{g(\ldots, \text{no } x_j, \ldots)} 
\]
But in the structural equation we would have that 
$x_i = f_{i}^1(K_i) + g_{i}^1(K_i)\eta_i$, hence both $f_i^1$ and $g_i^1$ 
depend on $x_j$. Even if $f$ cancels out the location component $f_i^1$, 
the numerator would still have $g_i^1$ depending on $x_j$ and the 
denominator does not depend on $x_j$:
\[
h(i,S) = \frac{g_i^1(\ldots, x_j, \ldots) \eta_i}
              {g(\ldots, \text{no } x_j, \ldots)} 
\]
Since $x_j = f^1_j(K_j) + g^1_j(K_j)\,\eta_j$ carries $\epsilon_j$ 
through $\eta_j$, and $g^1_i$ depends nonlinearly on $x_j$, by 
Lemma~\ref{lemma:non_decomp} the dependence on $\epsilon_j$ through $x_j$ 
cannot be factored out. Therefore $h(i, S) \notcind \epsilon_j$ for all 
choices of $f, g$ under this restriction. The case when $g$ uses $x_j$ but 
$f$ does not follows by a symmetric argument.

\noindent
\textbf{Part (b):} Both $f$ and $g$ do not use $x_j$:
\[
h(i,S) = \frac{x_i- f(\ldots,\text{no } x_j, \ldots)}
              {g(\ldots, \text{no } x_j, \ldots)} 
\]
But in the structural equation we would have that 
$x_i = f_{i}^1(K_i) + g_{i}^1(K_i)\eta_i$, hence both $f_i^1$ and $g_i^1$ 
depend on $x_j$. Neither $f$ nor $g$ depends on $x_j$, so neither can 
cancel these terms. Since $x_j = f_j^1(K_j) + g_j^1(K_j)\,\eta_j$ carries 
$\epsilon_j$, and the dependence on $x_j$ enters through the nonlinear 
functions $f_i^1$ and $g_i^1$, by Lemma~\ref{lemma:non_decomp}, the 
residual $h(i, S)$ retains dependence on $\epsilon_j$. Therefore 
$h(i,S) \notcind \epsilon_j$ for all choices of $f, g$ under this 
restriction.
\end{proof}

The following proposition provides an equivalent reformulation of the 
visible parent condition~\eqref{eq:case2b} that makes the regression sets 
explicit: $x_j$ must be part of the regression set for $x_i$ (i.e., 
available to \emph{both} the location function $f$ and the scale function 
$g$), and the regression sets for both $h_1$ and $h_2$ can be restricted 
to contain only parents of $x_i$ and $x_j$.

\begin{proposition}[Visible parent with explicit regression sets]%
\label{cor:parent-must-be-in-regression}
$x_j$ is a visible parent of $x_i$ if and only if \eqref{eq:case2a-main} 
holds and \eqref{eq:case2b-main} holds. Moreover, in any 
independence-achieving pair $h_1(i, \{i\}) \cind h_2(j, \{j\})$ 
from~\eqref{eq:case2b-main}, the functions $f$ and $g$ inside $h_1$ must 
both use $x_j$. That is, $x_j$ must be in the regression set of both 
the location function $f$ and the scale function $g$ in $h_1$.
\end{proposition}

\begin{proof}
The main content is showing that $x_j$ must be used by both $f$ and $g$ 
in any independence-achieving $h_1$.

Suppose for contradiction that there exist $h_1(i, \{i\})$ and 
$h_2(j, \{j\})$ achieving independence, where $x_j$ is excluded from the 
argument of at least one of $f$ or $g$ in $h_1$. 

By Proposition~\ref{prop:regression_sets_fg}(a,b), since $x_j$ is a 
parent of $x_i$ and is excluded from at least one of $f$ or $g$, the 
residual $h_1(i, \{i\})$ retains dependence on $\epsilon_j$:
\[
\epsilon_j \notcind h_1(i, \{i\}).
\]
By Lemma~\ref{lemma:noise-irremovability-LSNM-residuals} (irremovability 
of own noise), for any $h_2(j, \{j\})$:
\[
\epsilon_j \notcind h_2(j, \{j\}).
\]
By Assumption~\ref{assumption:residual_faithfulness} (residual 
faithfulness), since both $h_1(i, \{i\})$ and $h_2(j, \{j\})$ depend on 
the same exogenous noise $\epsilon_j$:
\[
h_1(i, \{i\}) \notcind h_2(j, \{j\}),
\]
contradicting the assumed independence. Therefore $x_j$ must be in the 
argument of both $f$ and $g$ in any independence-achieving $h_1$.
\end{proof}

\subsubsection{LSNM-UV-Base}
\label{base-algorithm-section}
\paragraph{Stage 1: LSNM-UV-Base.} This is the LSNM analogue of CAM-UV~\cite{maeda21a}, Algorithm~1. We adapt Algorithms~1 and~2 of \cite{maeda21a} from additive residuals to location-scale residuals. The algorithm has three phases: Phase~1 extracts parent candidates, Phase~2 prunes false candidates, and Phase~3 detects invisible pairs.

\begin{definition}[Sink]
Let $K \subseteq X$ be a set of observed variables and $x_i \in K$. We call $x_i$ a \emph{sink} of $K$ if no variable $x_j \in K \setminus \{x_i\}$ is a descendant of $x_i$ in~$G$.
\end{definition}

Throughout, $h(i, S)$ denotes a location-scale residual as defined earlier:
\[
  h(i, S) = \frac{x_i - f\!\bigl(\mathbf{x}_{X \setminus S_f}\bigr)}{g\!\bigl(\mathbf{x}_{X \setminus S_g}\bigr)}\,,
  \quad S_f \subseteq S,\; S_g \subseteq S,\;
  \neg(S_f \subsetneq S \wedge S_g \subsetneq S).
\]
The \emph{exclusion set} $S$ contains the indices removed from the argument of at least one of $f$ and $g$. The \emph{regression set} (what $f$ and $g$ can use) is $X \setminus S_f$ for the location and $X \setminus S_g$ for the scale.

To regress $x_i$ on a desired set $R \subseteq X \setminus \{x_i\}$, we set the exclusion set $S = X \setminus R$ (so that $i \in S$ since $x_i \notin R$), with both $S_f = S$ and $S_g = S$. In practice, the GAMLSS regression fits $f$ and $g$ simultaneously on the same regression set $R$, so both functions share the same inputs and $S_f = S_g = S$. Note that $i \in X \setminus R$ always holds since $x_i \notin R$.

When we write ``p-HSIC'' below, we mean the $p$-value of the gamma independence test based on the Hilbert--Schmidt Independence Criterion.

\begin{lemma}[Sink identification in LSNM-UV]
\label{lemma:sink}
Let $K \subseteq X$ with $x_i \in K$, and let $M_l \subseteq X \setminus K$ be a set of previously identified parent candidates for each $x_l$. If the following two conditions hold, then $x_i$ is a sink of $K$:
\begin{align}
&\exists\, h_i \in \mathcal{H}: \;\forall\, x_j \in K \setminus \{x_i\}, \;\exists\, h_j \in \mathcal{H}: \nonumber\\
&\qquad h_i\!\Big(i,\; X \setminus (M_i \cup K \setminus \{x_i\})\Big) \;\cind\; h_j\!\Big(j,\; X \setminus M_j\Big), \label{eq:sink_indep}\\[6pt]
&\forall\, x_j \in K \setminus \{x_i\}, \;\forall\, h_i', h_j' \in \mathcal{H}: \nonumber\\
&\qquad h_i'\!\Big(i,\; X \setminus M_i\Big) \;\notcind\; h_j'\!\Big(j,\; X \setminus M_j\Big). \label{eq:sink_dep}
\end{align}
\end{lemma}

\begin{proof}[Proof of Lemma~\ref{lemma:sink}]
We prove the contrapositive: if $x_i$ is not a sink of $K$, then at
least one of~\eqref{eq:sink_indep} or~\eqref{eq:sink_dep} fails.

Since $x_i$ is not a sink, there exists $x_j \in K \setminus \{x_i\}$
that is a descendant of $x_i$ in~$G$. The following three cases are
exhaustive: either $x_i$ is a direct parent of $x_j$ (Case~A), or
$A(j,i) = 0$ with $B(i,j) = 1$ (Case~B), or $A(j,i) = 0$ with
$B(i,j) = 0$ (Case~C).

\paragraph{Case~A: $x_i$ is a direct parent of $x_j$ ($A(j,i) = 1$).}
We show~\eqref{eq:sink_indep} fails.
Since $M_j \subseteq X \setminus K$ and $x_i \in K$, we have
$x_i \notin M_j$. This follows directly from the lemma hypothesis $M_j \subseteq X \setminus K$: since $x_i \in K$, $M_j$ and $K$ are disjoint, so $x_i \notin M_j$.
Since $x_i \in K_j$ is excluded from at least one of $f$ and $g$ in any $h_j(j, X \setminus M_j)$ (since $x_i \in X \setminus M_j$ and at least one of $S_f, S_g$ must equal $X \setminus M_j$),
Proposition~\ref{prop:regression_sets_fg}(a,b) gives:
\[
  \epsilon_i \notcind h_j(j, X \setminus M_j)
  \quad \text{for all } h_j \in \mathcal{H}.
\]
This holds because $x_i$ is a true parent of $x_j$, so any regression of $x_j$ that excludes $x_i$ retains its dependence on $\epsilon_i$, the exogenous noise of $x_i$.
By Lemma~\ref{lemma:noise-irremovability-LSNM-residuals}
(irremovability of own noise), $\epsilon_i \notcind h_i(i, S)$ for
all $h_i \in \mathcal{H}$ and any valid exclusion set~$S \ni i$.
In particular,
$\epsilon_i \notcind h_i\!\big(i,\; X \setminus (M_i \cup
K \setminus \{x_i\})\big)$ for all $h_i \in \mathcal{H}$.
Since both residuals depend on the same exogenous noise~$\epsilon_i$,
Assumption~\ref{assumption:residual_faithfulness} gives:
\[
  h_i\!\big(i,\; X \setminus (M_i \cup K \setminus \{x_i\})\big)
  \;\notcind\; h_j(j,\; X \setminus M_j)
  \quad \text{for all } h_i, h_j \in \mathcal{H}.
\]
No $h_i$ can achieve independence with~$h_j$,
so~\eqref{eq:sink_indep} fails.

\paragraph{Case~B: $B(i,j) = 1$ and $A(j,i) = 0$ (invisible pair).}
We show~\eqref{eq:sink_indep} fails.
Since $B(i,j) = 1$, there exists shared exogenous noise
$\epsilon_l \in N_i \cap N_j$
(Lemma~\ref{lemma:independence_patterns}). This noise enters both
$x_i$ and $x_j$ through hidden variables in~$Q_i$ and~$Q_j$
respectively. Since hidden variables are never available to any
regression function in~$\mathcal{H}$, the contribution of~$\epsilon_l$
cannot be cancelled by any choice of observed-variable regression, by
Lemma~\ref{lemma:non_decomp}. Therefore:
\[
  \epsilon_l \notcind h_i(i, S_1)
  \quad \text{and} \quad
  \epsilon_l \notcind h_j(j, S_2)
  \quad \text{for all } h_i, h_j \in \mathcal{H}
  \text{ and all valid } S_1, S_2.
\]
By Assumption~\ref{assumption:residual_faithfulness},
$h_i \notcind h_j$ for all $h_i, h_j \in \mathcal{H}$,
and~\eqref{eq:sink_indep} fails.

\paragraph{Case~C: $A(j,i) = 0$ and $B(i,j) = 0$ (indirect
descendant, no confounding).} We show~\eqref{eq:sink_dep} fails.
We first establish the following invariant, which the algorithm
maintains under bow-freeness:
\[
  K_\ell \subseteq M_\ell \quad \text{for every } x_\ell \in K.
\]
Parents of $x_\ell$ outside $K$ are covered by the standard inductive
hypothesis $K_\ell \setminus K \subseteq M_\ell$. For a parent
$x_k \in K_\ell \cap K$, bow-freeness gives
$A(\ell, k)\, B(\ell, k) = 0$; since $A(\ell, k) = 1$, we have
$B(\ell, k) = 0$, so $(x_\ell, x_k)$ is a visible pair. The pair
$\{x_\ell, x_k\}$ was processed in an earlier iteration with $t = 2$
(the reset rule $t \leftarrow 2$ in
Algorithm LSNM-UV-Base, ensures pair-level subsets
are revisited whenever a new parent is added). Since $A(\ell, k) = 1$
and $B(\ell, k) = 0$, by Case~A the sink test for $\{x_\ell, x_k\}$
correctly identifies $x_\ell$ as the sink, adding $x_k$ to $M_\ell$.
Hence $K_\ell \subseteq M_\ell$.

By acyclicity, $x_j$ being a descendant of $x_i$ rules out
$x_j \in K_i$. Symmetrically, $A(j,i) = 0$ gives $x_i \notin K_j$.
Combined with the invariant above, $K_i \subseteq M_i$ and
$K_j \subseteq M_j$.

Choose $h_i'$ using the true structural functions $f = f_i^1$,
$g = g_i^1$, with exclusion set $X \setminus M_i$ (so the regression
uses $M_i$). Since $K_i \subseteq M_i$, both $f_i^1$ and $g_i^1$ have
access to all true parents of $x_i$, and
Proposition~\ref{prop:regression_sets_fg}(c) gives:
\[
  h_i'(i,\; X \setminus M_i)
  = \frac{x_i - f_i^1(K_i)}{g_i^1(K_i)} = \eta_i.
\]
Symmetrically, $h_j'(j,\; X \setminus M_j) = \eta_j$. Since
$B(i,j) = 0$, Lemma~\ref{lemma:independence_patterns} gives
$\eta_i \cind \eta_j$. Therefore
$\exists\, h_i', h_j' \in \mathcal{H}$ with
\[
  h_i'(i,\; X \setminus M_i) \;\cind\; h_j'(j,\; X \setminus M_j),
\]
and~\eqref{eq:sink_dep} fails.


\end{proof}
\begin{algorithm}[H]
\footnotesize
\caption{LSNM-UV-Base}
\label{alg:lsnm-uv-base}
\begin{algorithmic}[1]
\Require Data matrix $\mathbf{X}$ for $p$ observed variables,
  max subset size $d$, significance level $\alpha$
\Ensure Adjacency matrix $A$, parent sets $\{M_1, \ldots, M_p\}$
 
\medskip
\State \textbf{Phase 1: Extract parent candidates}
\For{$i = 1$ to $p$}
  \State $M_i \leftarrow \emptyset$
\EndFor
\State $t \leftarrow 2$
\While{$t \leq d$}
  \State $\textit{noChange} \leftarrow \textbf{True}$
  \For{each $K \subseteq X$ with $|K| = t$}
 
    \medskip
    \State \textbf{Step A: Find the most endogenous variable in $K$}
    \For{each $x_i \in K$}
      \State $\mathrm{score}(x_i)
        \leftarrow \text{p-HSIC}\!\bigg(
          h_i\!\Big(i,\; X \setminus \bigl(M_i \cup K \setminus \{x_i\}\bigr)\Big),\;
          \Big\{h_j\!\bigl(j,\; X \setminus M_j\bigr)\Big\}_{x_j \in K \setminus \{x_i\}}
        \bigg)$
    \EndFor
    \State $x_b \leftarrow \arg\max_{x_i \in K} \mathrm{score}(x_i)$
 
    \medskip
    \State \textbf{Step B: Verify $x_b$ is a sink (Lemma~\ref{lemma:sink})}
    \State $e \leftarrow \text{p-HSIC}\!\bigg(
      h_b\!\Big(b,\; X \setminus \bigl(M_b \cup K \setminus \{x_b\}\bigr)\Big),\;
      \Big\{h_j\!\bigl(j,\; X \setminus M_j\bigr)\Big\}_{x_j \in K \setminus \{x_b\}}
    \bigg)$
    \State $d_{\max} \leftarrow \displaystyle\max_{x_j \in K \setminus \{x_b\}}
      \text{p-HSIC}\!\bigg(
        h_b\!\bigl(b,\; X \setminus M_b\bigr),\;
        h_j\!\bigl(j,\; X \setminus M_j\bigr)
      \bigg)$
    \If{$e > \alpha$ \textbf{and} $d_{\max} < \alpha$}
      \State $M_b \leftarrow M_b \cup (K \setminus \{x_b\})$
        \Comment{$x_b$ is sink; others are parent candidates}
      \State $\textit{noChange} \leftarrow \textbf{False}$
    \EndIf
 
  \EndFor
  \If{$\textit{noChange}$}
    \State $t \leftarrow t + 1$
  \Else
    \State $t \leftarrow 2$
      \Comment{Restart: new parents may unlock new sinks}
  \EndIf
\EndWhile
 
\medskip
\State \textbf{Phase 2: Prune false parent candidates}
\For{$i = 1$ to $p$}
  \For{each $x_j \in M_i$}
    \If{$\text{p-HSIC}\!\bigg(
         h_i\!\Big(i,\; X \setminus \bigl(M_i \setminus \{x_j\}\bigr)\Big),\;
         h_j\!\bigl(j,\; X \setminus M_j\bigr)
       \bigg) > \alpha$}
      \State $M_i \leftarrow M_i \setminus \{x_j\}$
        \Comment{$x_j$ not a direct parent}
    \EndIf
  \EndFor
\EndFor
 
\medskip
\State \textbf{Phase 3: Build adjacency matrix and detect invisible pairs}
\State Initialize $A(i,j) \leftarrow 0$ for all $i, j$
\For{$i = 1$ to $p$}
  \For{each $x_j \in M_i$}
    \State $A(i,j) \leftarrow 1$
      \Comment{$x_j$ is a parent of $x_i$}
  \EndFor
\EndFor
\For{each pair $(i,j)$ with $i < j$,\;
     $x_j \notin M_i$,\; $x_i \notin M_j$}
  \If{$\text{p-HSIC}\!\bigg(
       h_i\!\bigl(i,\; X \setminus M_i\bigr),\;
       h_j\!\bigl(j,\; X \setminus M_j\bigr)
     \bigg) < \alpha$}
    \State $A(i,j) \leftarrow \mathrm{NaN};\; A(j,i) \leftarrow \mathrm{NaN}$
      \Comment{Invisible pair: UBP/UCP exists}
  \EndIf
\EndFor
 
\medskip
\Return $A$, $\{M_1, \ldots, M_p\}$
\end{algorithmic}
\end{algorithm}


\begin{itemize}
\item \textbf{Lines 8--13 (Step A):} We start with all subsets $K$ of size $t$ contained in $X$. To find a \emph{sink} for the subset $K$, we first find the most endogenous element in $K$. For each element $x_i$ in $K$, we compute a score: after regressing $x_i$ on its parent candidates $M_i$ and all other members of the subset $K$, is $x_i$'s residual jointly independent of the residuals of all other members of $K$, each regressed on their own parent candidates? A high score (high p-value) means independence, and hence $x_b$ is chosen as the element achieving the maximum score. When a sink is confirmed in Step~B, the algorithm resets $t \leftarrow 2$ (line~22): the newly discovered parents in $M_b$ may unlock sinks in other subsets that previously failed the test, so the search restarts from pairs.

\item \textbf{Lines 14--17 (Step B):} We check if $x_b$ is indeed a sink. Condition~\eqref{eq:sink_indep} (independence with $K$ included) is checked in line~15, and condition~\eqref{eq:sink_dep} (dependence without $K$) is checked in line~16. In line~17, if $e > \alpha$ (independence) and $d_{\max} < \alpha$ (dependence) are both verified, $x_b$ is confirmed as a sink and $K \setminus \{x_b\}$ is added to $M_b$.

\item \textbf{Lines 28--36 (Phase 2):} The pruning test checks whether $h_i(i, X \setminus (M_i \setminus \{x_j\})) \cind h_j(j, X \setminus M_j)$. Note that $X \setminus (M_i \setminus \{x_j\})$ contains $j$ (since $x_j$ was removed from $M_i$), so $x_j$ is in the exclusion set of~$h_i$. In practice, since we set $S_f = S_g = X \setminus (M_i \setminus \{x_j\})$, $x_j$ is excluded from both $f$ and $g$. By \textbf{Proposition~\ref{prop:regression_sets_fg}(b)}, if $x_j$ is a true parent of $x_i$, then excluding $x_j$ from both $f$ and $g$ forces the residual to retain dependence on $\epsilon_j$. Combined with \textbf{Lemma~\ref{lemma:noise-irremovability-LSNM-residuals}} (irremovability of own noise) and \textbf{Assumption~\ref{assumption:residual_faithfulness}} (residual faithfulness), this guarantees $h_i \notcind h_j$ for all $h_i, h_j \in \mathcal{H}$, so true parents are never pruned.

Conversely, if $x_j$ is not a direct parent, then $M_i \setminus \{x_j\}$ still contains all true parents $K_i$. The true structural functions $f_i^1, g_i^1$ use only $K_i \subseteq M_i \setminus \{x_j\}$, so setting $S_f = S_g = X \setminus (M_i \setminus \{x_j\})$ is valid for the true functions. This gives $h_i(i, X \setminus (M_i \setminus \{x_j\})) = \eta_i$. Similarly, $h_j(j, X \setminus M_j) = \eta_j$. Note that any $x_j \in M_i$ that is not a true parent must form a visible non-edge with $x_i$ (invisible pair partners cannot enter $M_i$ because the sink test, which requires joint independence, always fails when an invisible partner is in $K$). Since the pair is visible (no UBP/UCP), \textbf{Lemma~\ref{lemma:independence_patterns}} gives $\eta_i \cind \eta_j$. Non-parents are correctly pruned.

\item \textbf{Lines 38--48 (Phase 3):} For each non-adjacent pair $(x_i, x_j)$, the test $h_i(i, X \setminus M_i) \notcind h_j(j, X \setminus M_j)$ detects invisible pairs. When $M_i = K_i$ (the true parents), the true structural functions give $h_i(i, X \setminus M_i) = \eta_i$, and \textbf{Lemma~\ref{lemma:independence_patterns}} gives $\eta_i \notcind \eta_j$ if and only if a UBP/UCP exists. This corresponds to \textbf{Proposition~\ref{prop:residual_independence_pattern}}, Case~1 (bidirected edge).

However, when $M_i \subsetneq K_i$ (some true parents were not identified in Phase~1, e.g., because invisible pairs blocked sink identification), the residual $h_i(i, X \setminus M_i)$ does not equal $\eta_i$ and may show dependence with $h_j$ even when no UBP/UCP exists. This causes visible pairs to be incorrectly marked as invisible. This is a known limitation, discussed further below.
\end{itemize}


\subsubsection{Limitations of LSNM-UV-Base}
\label{sec:limitations_base}
LSNM-UV-Base is \emph{incomplete} in identifying visible pairs and \emph{unsound} in identifying invisible pairs. Soundness and completeness can be defined as follows.

\begin{definition}[Soundness and completeness]\label{def:soundness_completeness}
\begin{itemize}
    \item The algorithm is \emph{sound for visible edges} if every visible edge in the output is also a visible edge in the ground truth.
    \item The algorithm is \emph{complete for visible edges} if every visible edge in the ground truth also appears as a visible edge in the output.
    \item Analogous definitions apply to visible non-edges and invisible pairs.
\end{itemize}
\end{definition}

The source of these limitations is that Phase~1 may fail to identify all parents of $x_i$. Specifically, the sink identification in Lemma~\ref{lemma:sink} requires that~\eqref{eq:sink_indep} holds: including $K \setminus \{x_i\}$ in the regression achieves independence. If some $x_j \in K$ has a UBP/UCP with $x_i$ (i.e., the pair is invisible), the shared hidden noise persists in the residuals regardless of the regression set, so~\eqref{eq:sink_indep} fails. The sink test does not pass, and the invisible parent is never added to $M_i$.



\begin{remark}
These limitations are the same as those of CAM-UV~\citep{maeda21a} in the additive setting, and are not specific to the LSNM extension. They arise because Phase~1 cannot identify invisible parents, not because of any issue with the location-scale structure.
\end{remark}

\subsubsection{CheckVisible}
\label{checkvisible-section}
\paragraph{Stage 2: checkVisible.} For each pair $(x_i, x_j)$ marked as invisible ($\mathrm{NaN}$) by LSNM-UV-Base, this procedure tests whether it can be re-classified as a visible edge or visible non-edge. The key insight is that LSNM-UV-Base searches over regression sets of the form $M_i$ (parent candidates found in Phase~1), which may be incomplete. \textsc{checkVisible} searches over a larger set $Q$ that includes all parents identified by LSNM-UV-Base \emph{plus} all variables whose relationship to $x_i$ or $x_j$ is unclear (marked NaN). This broader search can discover independence-achieving regression sets that Phase~1 missed.

For this procedure, we need the following lemma, which provides an alternative way to certify that a pair is not invisible and that one variable is not a parent of the other:
\begin{lemma}[Non-parentship and non-invisibility via inclusion]
\label{lem:nonparent-inclusion}
If
\begin{equation}\label{eq:lem9-a}
  \exists\, h_1, h_2 \in \mathcal{H},\;
  S_1, S_2 \subseteq X
  \text{ with } i \in S_1,\; x_j \notin S_1,\; \{i,j\} \subseteq S_2 :\quad
  h_1(i,\, S_1) \;\cind\; h_2(j,\, S_2),
\end{equation}
then $(x_i, x_j)$ is not invisible, and $x_i$ is not a parent of $x_j$.

Similarly, if
\begin{equation}\label{eq:lem9-b}
  \exists\, h_1, h_2 \in \mathcal{H},\;
  S_1, S_2 \subseteq X
  \text{ with } \{i,j\} \subseteq S_1,\; x_i \notin S_2,\; j \in S_2 :\quad
  h_1(i,\, S_1) \;\cind\; h_2(j,\, S_2),
\end{equation}
then $(x_i, x_j)$ is not invisible, and $x_j$ is not a parent of $x_i$.

If both~\eqref{eq:lem9-a} and~\eqref{eq:lem9-b} hold,
then $(x_i, x_j)$ is a visible non-edge.
\end{lemma}

\begin{proof}
We prove~\eqref{eq:lem9-a}; the proof of~\eqref{eq:lem9-b} is symmetric.

\noindent\textbf{$(x_i, x_j)$ is not invisible.}
Suppose for contradiction that $(x_i, x_j)$ is invisible, i.e., $B(i,j) = 1$.
By Lemma~\ref{lemma:independence_patterns}, there exists a shared exogenous noise $\epsilon_l \in N_i \cap N_j$.

Since $\epsilon_l$ enters $x_i$ through hidden variables in $Q_i$ and the regression in $h_1$ uses only observed variables, $\epsilon_l \notcind h_1(i, S_1)$ by Lemma~\ref{lemma:non_decomp}. Since $j \in S_2$ (from $\{i,j\} \subseteq S_2$), the same argument gives $\epsilon_l \notcind h_2(j, S_2)$. By residual faithfulness (Assumption~\ref{assumption:residual_faithfulness}), $h_1 \notcind h_2$. But~\eqref{eq:lem9-a} gives independence. Contradiction.

\noindent\textbf{$x_i$ is not a parent of $x_j$.}
Suppose for contradiction that $x_i \to x_j$ is a visible edge, i.e., $A(j,i) = 1$ and $B(j,i) = 0$.

First, by Lemma~\ref{lemma:noise-irremovability-LSNM-residuals}, $\epsilon_i \notcind h_1(i, S_1)$ for any valid $h_1 \in \mathcal{H}$, regardless of regression set. 

Second, $\epsilon_i \notcind h_2(j, S_2)$: since $x_i \to x_j$, $x_i \in K_j$; and $i \in S_2$ (from $\{i,j\} \subseteq S_2$), so $x_i$ is excluded from at least one of $f$ and $g$ in $h_2$. By Proposition~\ref{prop:regression_sets_fg}(a,b), excluding a true parent from at least one of $f$ and $g$ forces the residual to retain dependence on $\epsilon_i$.

By residual faithfulness (Assumption~\ref{assumption:residual_faithfulness}), $h_1(i, S_1) \notcind h_2(j, S_2)$. But~\eqref{eq:lem9-a} gives independence. Contradiction.

\noindent\textbf{Visible non-edge when both hold.}
If both~\eqref{eq:lem9-a} and~\eqref{eq:lem9-b} hold, then $(x_i, x_j)$ is not invisible. One of three cases must hold: (1)~$x_i$ is a visible parent of $x_j$, (2)~$x_j$ is a visible parent of $x_i$, or (3)~visible non-edge. But~\eqref{eq:lem9-a} rules out case~(1), and~\eqref{eq:lem9-b} rules out case~(2). Therefore $(x_i, x_j)$ is a visible non-edge.
\end{proof}

\begin{algorithm}[H]
\caption{\textsc{checkVisible LSNM-UV}}
\label{alg:checkvisible}
\begin{algorithmic}[1]
\Require Indices $i$ and $j$
\State $P_i \leftarrow \{v \mid A(i,v) = 1\}$;\quad
       $P_j \leftarrow \{v \mid A(j,v) = 1\}$
\State $Q \leftarrow P_i \cup P_j \cup \{k \mid A(i,k) = \mathrm{NaN} \text{ or } A(j,k) = \mathrm{NaN}\}$
\State $\textit{iNotParent} \leftarrow \textbf{False}$;\quad
       $\textit{jNotParent} \leftarrow \textbf{False}$
 
\For{each $M \subseteq Q$ and $N \subseteq Q$}
  \State $e \leftarrow \text{p-HSIC}\!\Big(
    h_i\!\bigl(i,\; X \setminus M\bigr),\;
    h_j\!\bigl(j,\; X \setminus N\bigr)
  \Big)$
  \State \Comment{Test visible non-edge: Prop~\ref{prop:residual_independence_pattern}, Case~3}
  \If{$e > \alpha$}
    \State $A(i,j) \leftarrow 0$;\; $A(j,i) \leftarrow 0$;\; \textbf{return}
  \Else
    \State $a_1 \leftarrow \text{p-HSIC}\!\Big(
      h_i\!\bigl(i,\; X \setminus (M \cup \{x_j\})\bigr),\;
      h_j\!\bigl(j,\; X \setminus N\bigr)
    \Big)$
    \State $a_2 \leftarrow \text{p-HSIC}\!\Big(
      h_i\!\bigl(i,\; X \setminus M\bigr),\;
      h_j\!\bigl(j,\; X \setminus (N \cup \{x_i\})\bigr)
    \Big)$
    \If{$a_1 > \alpha$}
      \State $\textit{iNotParent} \leftarrow \textbf{True}$
        \Comment{$x_i$ is not a parent of $x_j$}
    \EndIf
    \If{$a_2 > \alpha$}
      \State $\textit{jNotParent} \leftarrow \textbf{True}$
        \Comment{$x_j$ is not a parent of $x_i$}
    \EndIf
    \If{$\textit{iNotParent} \wedge \textit{jNotParent}$}
      \State $A(i,j) \leftarrow 0$;\; $A(j,i) \leftarrow 0$;\; \textbf{return}
        \Comment{Visible non-edge}
    \EndIf
  \EndIf
\EndFor
 
\If{$\textit{iNotParent}$}
  \State $A(i,j) \leftarrow 1$;\; $A(j,i) \leftarrow 0$
    \Comment{$x_j \to x_i$ is a visible edge}
\EndIf
\If{$\textit{jNotParent}$}
  \State $A(j,i) \leftarrow 1$;\; $A(i,j) \leftarrow 0$
    \Comment{$x_i \to x_j$ is a visible edge}
\EndIf
\end{algorithmic}
\end{algorithm}

For a pair $(x_i, x_j)$, the search set $Q$ consists of all parents identified by LSNM-UV-Base, together with all variables whose relationship to $x_i$ or $x_j$ is invisible. Note that $Q$ does not contain $x_i$ or $x_j$. Lines~10--11 search for exclusion sets $S_1$ and $S_2$ satisfying the conditions of Lemma~\ref{lem:nonparent-inclusion}: iterating over $M, N \subseteq Q$, the algorithm sets $S_1 = X \setminus (M \cup \{x_j\})$ (so $i \in S_1$ and $x_j \notin S_1$, i.e., $x_j$ is available in $x_i$'s regression) and $S_2 = X \setminus N$ (so $\{i,j\} \subseteq S_2$ since $Q \subseteq X \setminus \{x_i,x_j\}$), matching~\eqref{eq:lem9-a}; and symmetrically for~\eqref{eq:lem9-b}. For all subsets of $Q$, we check the following:
\begin{itemize}
\item \textbf{Line 7} ($e > \alpha$):
Tests \textbf{Proposition~\ref{prop:residual_independence_pattern}}, Case~3 (visible non-edge), \eqref{eq:case3}.
If $\exists\, h_i, h_j \in \mathcal{H}$:
$h_i(i, \{i,j\}) \cind h_j(j, \{i,j\})$,
then $(x_i, x_j)$ is a visible non-edge.

\item \textbf{Line 10} ($a_1 > \alpha$):
  Tests \textbf{Lemma~\ref{lem:nonparent-inclusion}}, Equation~\eqref{eq:lem9-a}:
  $\exists\, S_1, S_2$ with $i \in S_1$, $x_j \notin S_1$, $\{i,j\} \subseteq S_2$:
  $h_1(i, S_1) \cind h_2(j, S_2)$.
  Then the pair is not invisible and $x_i$ is not a parent of $x_j$.

\item \textbf{Line 11} ($a_2 > \alpha$):
  Symmetric test for Equation~\eqref{eq:lem9-b}:
  $\exists\, S_1, S_2$ with $\{i,j\} \subseteq S_1$, $x_i \notin S_2$, $j \in S_2$:
  $h_1(i, S_1) \cind h_2(j, S_2)$.
  Then the pair is not invisible and $x_j$ is not a parent of $x_i$.
 
\item \textbf{Line 17}:
  If neither is a parent of the other, then $(x_i, x_j)$ is a visible non-edge
  (\textbf{Lemma~\ref{lem:nonparent-inclusion}}).
 
\item \textbf{Lines 21--24}:
  If $\textit{iNotParent}$ but not $\textit{jNotParent}$:
  universal dependence (\textbf{Proposition~\ref{prop:residual_independence_pattern}}, Case~2, \eqref{eq:case2a}) holds for
  all $h_i(i, \{i,j\})$ and $h_j(j, \{j\})$ (since line~7 was never triggered for any $M, N \subseteq Q$),
  and existence of independence (\eqref{eq:case2b}) is witnessed by $a_1 > \alpha$. Together these imply $x_j$ is a visible parent of $x_i$.
\end{itemize}


\subsubsection{Soundness and completeness of LSNM-UV}
\label{sound-complete-proof}
Before proving the theorem \ref{thm:soundness_completeness}, we establish that non-parent variables can be removed from regression sets without affecting independence. These are needed for the completeness arguments.

\begin{proposition}[Non-parent removal for visible parent regression sets]
\label{prop:remove_vis_parent}
Assume $x_j$ is a visible parent of $x_i$, and $x_k$ is not a parent of $x_i$ and not a parent of $x_j$. If
$\exists\, h_1, h_2 \in \mathcal{H}$: $h_1(i, \{i\}) \cind h_2(j, \{j\})$
where $h_1$ or $h_2$ uses $x_k$ in its regression, then there exist $h_1', h_2'$ that do NOT use $x_k$ and still achieve independence.
\end{proposition}
\begin{proof}
Since $x_k$ is not a parent of $x_i$, the true structural functions $f_i^1$ and $g_i^1$ do not depend on $x_k$. Since $x_k$ is not a parent of $x_j$, the true structural functions $f_j^1$ and $g_j^1$ do not depend on $x_k$.

Choose $h'_1$ using the true structural functions for $x_i$: $f = f_i^1$ and $g = g_i^1$. Set $S_f = S_g = \{i\}$. Since $x_i \notin K_i$ (DAG property), $K_i \subseteq X \setminus \{x_i\}$, so both $f_i^1$ and $g_i^1$ use only variables in $X \setminus \{x_i\} = X \setminus S_f$, satisfying the constraint for $S = \{i\}$. Since $f_i^1$ and $g_i^1$ do not use $x_k$, $h'_1$ does not use $x_k$. This gives:
\[
h'_1(i, \{i\}) = \frac{x_i - f_i^1(K_i)}{g_i^1(K_i)} = \eta_i.
\]

Choose $h'_2$ using the true structural functions for $x_j$: $f = f_j^1$ and $g = g_j^1$. Set $S_f = S_g = \{j\}$. Since $x_j \notin K_j$ (DAG property), $K_j \subseteq X \setminus \{x_j\}$, satisfying the constraint for $S = \{j\}$. Since $f_j^1$ and $g_j^1$ do not use $x_k$, $h'_2$ does not use $x_k$. This gives $h'_2(j, \{j\}) = \eta_j$.

Since $B(i,j) = 0$ (the pair is visible), Lemma~\ref{lemma:independence_patterns} gives $\eta_i \cind \eta_j$. Therefore $h'_1(i, \{i\}) \cind h'_2(j, \{j\})$, and neither $h'_1$ nor $h'_2$ uses $x_k$.
\end{proof}

\begin{proposition}[Non-parent removal for visible non-edge regression sets]
\label{prop:remove_vis_non_edge}
Assume $(x_i,x_j)$ is a visible non-edge, and $x_k$ is not a parent of $x_i$ and not a parent of $x_j$. If
$\exists\, h_1, h_2 \in \mathcal{H}$: $h_1(i, \{i, j\}) \cind h_2(j, \{i,j\})$
where $h_1$ or $h_2$ uses $x_k$ in its regression, then there exist $h_1', h_2'$ that do NOT use $x_k$ and still achieve independence.
\end{proposition}
\begin{proof}
Since $x_k$ is not a parent of $x_i$ or $x_j$, the true structural functions $f_i^1, g_i^1, f_j^1, g_j^1$ do not depend on $x_k$.

Since $A(i,j) = 0$ (visible non-edge), $x_j$ is not a parent of $x_i$, and since $x_i$ cannot be its own parent, $K_i \subseteq X \setminus \{x_i, x_j\}$. Choose $h'_1$ using the true structural functions for $x_i$ with $S_f = S_g = \{i,j\}$. Since $K_i \subseteq X \setminus \{x_i, x_j\} = X \setminus S_f$, both $f_i^1$ and $g_i^1$ use only variables in $X \setminus S_f$, satisfying the constraint for $S = \{i,j\}$. Since $f_i^1$ and $g_i^1$ do not use $x_k$, $h'_1$ does not use $x_k$. This gives $h'_1(i, \{i,j\}) = \eta_i$.

Similarly, since $A(j,i) = 0$, $K_j \subseteq X \setminus \{x_i, x_j\}$. Choosing the true structural functions for $x_j$ with $S_f = S_g = \{i,j\}$ gives $h'_2(j, \{i,j\}) = \eta_j$, and $h'_2$ does not use $x_k$.

Since $B(i,j) = 0$ (no UBP/UCP), Lemma~\ref{lemma:independence_patterns} gives $\eta_i \cind \eta_j$. Therefore $h'_1(i, \{i,j\}) \cind h'_2(j, \{i,j\})$, and neither uses $x_k$.
\end{proof}


\begin{proof}[Proof of Theorem~\ref{thm:soundness_completeness}]
Throughout this proof, we use two properties of LSNM-UV-Base (established in Section~\ref{sec:limitations_base}, following~\cite{maeda21a}):
\begin{enumerate}
\item[(S)] \emph{Soundness for visible pairs:} every visible edge and every visible non-edge in the output of LSNM-UV-Base is also a visible edge (resp.\ visible non-edge) in the ground truth.
\item[(Q)] \emph{Parent coverage of $Q$:} for any pair $(x_i,x_j)$ processed by \textsc{checkVisible}, the search set $Q = P_i \cup P_j \cup \{x_k : A(i,k) = \mathrm{NaN} \text{ or } A(j,k) = \mathrm{NaN}\}$ contains every parent of $x_i$ and every parent of $x_j$. This holds because any true parent $x_k$ of $x_i$ is either in $P_i$ (identified by LSNM-UV-Base) or forms an invisible pair with~$x_i$ (marked NaN); it cannot be classified as a visible non-edge by~(S). The same argument applies to parents of~$x_j$.
\end{enumerate}

\noindent\textbf{Soundness for identifying visible non-edges.}
A pair $(x_i, x_j)$ is identified as a visible non-edge in LSNM-UV-X if (a)~it is identified as a visible non-edge by LSNM-UV-Base, or (b)~it is identified as invisible by LSNM-UV-Base and re-identified as a visible non-edge by \textsc{checkVisible}.

For case~(a), by property~(S), $(x_i, x_j)$ is also a visible non-edge in the ground truth.

For case~(b), this means either: (1)~$e > \alpha$ in line~7 of \textsc{checkVisible}, or (2)~$\textit{iNotParent} = \textbf{True}$ and $\textit{jNotParent} = \textbf{True}$ in line~17.

If $e > \alpha$, then by Assumption~\ref{ass:HSIC}, $\exists\, h_i, h_j \in \mathcal{H}$: $h_i(i, \{i,j\}) \cind h_j(j, \{i,j\})$. By the backward direction of Proposition~\ref{prop:residual_independence_pattern}, Case~3, $A(i,j) = 0 \wedge A(j,i) = 0 \wedge B(i,j) = 0$, so $(x_i, x_j)$ is a visible non-edge in the ground truth.
    
 If $\textit{iNotParent} = \textbf{True}$ and $\textit{jNotParent} = \textbf{True}$, then by Assumption~\ref{ass:HSIC}, Eqs.~\eqref{eq:lem9-a} and~\eqref{eq:lem9-b} are satisfied. By Lemma~\ref{lem:nonparent-inclusion}, $(x_i, x_j)$ is a visible non-edge in the ground truth.

Therefore, LSNM-UV is sound in identifying visible non-edges.
\\
\textbf{Completeness for identifying visible non-edges.}
Suppose that $(x_i, x_j)$ is a visible non-edge in the ground truth. By the forward direction of Proposition~\ref{prop:residual_independence_pattern}, Case~3, $\exists\, h_1, h_2 \in \mathcal{H}$: $h_1(i, \{i,j\}) \cind h_2(j, \{i,j\})$.

After the execution of LSNM-UV-Base, there are three cases:

$(x_i, x_j)$ is concluded as a visible non-edge: LSNM-UV leaves the pair as is. Done.

$(x_i, x_j)$ is concluded as a visible edge: this does not happen, by property~(S).

$(x_i, x_j)$ is concluded as an invisible pair: \textsc{checkVisible} is executed for $(i, j)$. By property~(Q), $Q$ contains every parent of $x_i$ and~$x_j$. Since~\eqref{eq:case3} holds, there exist some sets $M \subseteq X \setminus \{x_i, x_j\}$ and $N \subseteq X \setminus \{x_i, x_j\}$ that realize the independence between residuals, and produce a p-HSIC value greater than~$\alpha$ by Assumption~\ref{ass:HSIC}.

 If $M$ and $N$ do not contain non-parents of $x_i$ and $x_j$, this means $M, N \subseteq Q$, and exhaustively searching through $Q$ as in \textsc{checkVisible} ensures the finding of $M$ and $N$.
    
If $M$ or $N$ contain some non-parent $x_k$ of $x_i$ and $x_j$, Proposition~\ref{prop:remove_vis_non_edge} implies that there exist some sets $M'$ and $N'$ that do not use $x_k$ such that the independence between residuals is realized. Since the true structural functions do not depend on non-parents, the construction in Proposition~\ref{prop:remove_vis_non_edge} guarantees $M', N'$ can be taken as subsets of the parents of $x_i$ and $x_j$, hence $M', N' \subseteq Q$. Exhaustively searching through $Q$ as in \textsc{checkVisible} ensures the finding of $M'$ and $N'$, and produces a p-HSIC value greater than $\alpha$ by Assumption~\ref{ass:HSIC}.

Therefore, \textsc{checkVisible} is guaranteed to find some set $M, N \subseteq Q$ such that $e > \alpha$ and line~7 is satisfied. Thus, LSNM-UV outputs $(x_i, x_j)$ as a visible non-edge.

Therefore, LSNM-UV is complete in identifying visible non-edges.
\\
\textbf{Soundness for identifying visible edges.}
An edge $x_j \to x_i$ is identified as visible in LSNM-UV if (a)~it is identified as visible by LSNM-UV-Base, or (b)~it is identified as invisible by LSNM-UV-Base and re-identified as visible by \textsc{checkVisible}.

For case~(a), by property~(S), $x_j \to x_i$ is also a visible edge in the ground truth.

For case~(b), this means that $e \leq \alpha$ in line~7 for all sets $M, N \subseteq Q$, $a_1 > \alpha$ (since $\textit{iNotParent}$ must be $\textbf{True}$) for some set $M, N \subseteq Q$, and $a_2 \leq \alpha$ (since $\textit{jNotParent}$ must remain $\textbf{False}$) for all sets $M, N \subseteq Q$ in \textsc{checkVisible}.

Since $a_1 > \alpha$ for some set $M, N \subseteq Q$, by Assumption~\ref{ass:HSIC}, $\exists\, h_i, h_j$: $h_i(i, S)$ with $x_j$ in the regression achieves independence with $h_j(j, S')$. This satisfies~\eqref{eq:case2b}.
    
Since $e \leq \alpha$ and $a_2 \leq \alpha$ for all sets checked, by Assumption~\ref{ass:HSIC}, no functions in $\mathcal{H}$ achieve independence for any $M \subseteq Q$ and $N \subseteq Q \cup \{x_i\}$. By property~(Q), $Q$ contains every parent of $x_i$ and~$x_j$, so this implies~\eqref{eq:case2a}.

Since both~\eqref{eq:case2a} and~\eqref{eq:case2b} are satisfied, $x_j$ is a visible parent of $x_i$ in the ground truth by Proposition~\ref{prop:residual_independence_pattern}, Case~2.

Therefore, LSNM-UV-X is sound in identifying visible edges.
\\
\textbf{Completeness for identifying visible edges.}
Suppose that $x_j \to x_i$ is a visible edge in the ground truth. By the forward direction of Proposition~\ref{prop:residual_independence_pattern}, Case~2, both~\eqref{eq:case2a} and~\eqref{eq:case2b} hold.

After the execution of LSNM-UV-Base, there are three cases:

$x_j \to x_i$ is concluded as a visible edge: LSNM-UV leaves the edge as is. Done.

$x_j \to x_i$ is concluded as a visible non-edge: this does not happen, by property~(S).

$x_j \to x_i$ is concluded as an invisible pair: \textsc{checkVisible} is executed for $(i, j)$. Due to~\eqref{eq:case2a} and Assumption~\ref{ass:HSIC}, $e \leq \alpha$ and $a_2 \leq \alpha$ for all sets $M, N \subseteq Q$. Thus, $\textit{jNotParent}$ remains $\textbf{False}$, and lines~7 and~17 of \textsc{checkVisible} are never executed. By property~(Q), $Q$ contains every parent of $x_i$ and~$x_j$. Since~\eqref{eq:case2b} holds, there exist some sets $M$ and $N$ that realize the independence between residuals.

If $M$ and $N$ do not contain non-parents of $x_i$ and $x_j$, this means $M, N \subseteq Q$, and searching through $Q$ ensures the finding of $M$ and $N$.
    
If $M$ or $N$ contain some non-parent $x_k$ of $x_i$ and $x_j$, Proposition~\ref{prop:remove_vis_parent} implies that there exist some sets $M'$ and $N'$ that do not use $x_k$ such that the independence between residuals is realized. Since the true structural functions do not depend on non-parents, the construction in Proposition~\ref{prop:remove_vis_parent} guarantees $M', N'$ can be taken as subsets of the parents of $x_i$ and $x_j$, hence $M', N' \subseteq Q$. Searching through $Q$ ensures the finding of $M'$ and $N'$.

Therefore, LSNM-UV is guaranteed to find some sets such that regression in those sets will produce $a_1 > \alpha$ in line~10 of \textsc{checkVisible}, due to Assumption~\ref{ass:HSIC}. Therefore, $\textit{iNotParent}$ will be changed to $\textbf{True}$. Coupling this with the fact that $\textit{jNotParent}$ remains $\textbf{False}$, and the fact that lines 7 and 17 of \textsc{checkVisible} are not executed, one can conclude that line~21 of \textsc{checkVisible} is guaranteed to be executed. Thus, LSNM-UV outputs $x_j \to x_i$ as a visible edge.
Therefore, LSNM-UV is complete in identifying visible edges.
\\
\textbf{Soundness for identifying invisible pairs.}
When LSNM-UV identifies a pair $(x_i, x_j)$ as invisible, the pair cannot be a visible edge in the ground truth, since this would contradict the proven completeness of LSNM-UV in identifying visible edges. It also cannot be a visible non-edge in the ground truth, since this would contradict the proven completeness of LSNM-UV in identifying visible non-edges. Therefore, the pair must be invisible in the ground truth. This means LSNM-UV is sound in identifying invisible pairs.
\\
\textbf{Completeness for identifying invisible pairs.}
Suppose $(x_i, x_j)$ is an invisible pair in the ground truth. In the output of LSNM-UV, the pair cannot be a visible edge, since this would contradict the proven soundness of LSNM-UV in identifying visible edges. It also cannot be a visible non-edge in the output, since this would contradict the proven soundness of LSNM-UV in identifying visible non-edges. Therefore, the pair must be invisible in the output of LSNM-UV. This means LSNM-UV is complete in identifying invisible pairs.

\end{proof}


\subsubsection{Identifying ancestors on invisible paths}
\label{sec:checkonpath}

Lemma~4 of \cite{pham2025} provides conditions under which observed variables
can be identified as lying on invisible paths (UBPs/UCPs) between a visible
parent pair. That result is stated using the notion of visibility with respect
to subsets $X' \subseteq X$, which relies on the additive model being closed
under marginalization of observed variables (see
Remark~\ref{rem:no_subset_visibility}). We now provide an LSNM analogue that
avoids subset visibility entirely, stating all conditions directly in terms of
GAMLSS residuals $h(i, S)$ on the full data.



\begin{proof}[Proof of Lemma~\ref{lem:checkonpath}]
We prove the contrapositive: if $x_k \notin \mathrm{Anc}(x_i)$, then
independence is achievable in~\eqref{eq:cop-necessary1}.

\paragraph{Step~1: $x_k \notin K_i$ and $x_k \notin K_j$.}
Since $x_k \notin \mathrm{Anc}(x_i)$ and every parent is an ancestor,
$x_k$ is not a parent of $x_i$, so $x_k \notin K_i$.

Since $x_j \to x_i$ is a directed edge, every ancestor of $x_j$ is also
an ancestor of $x_i$ (otherwise the graph would contain a cycle).
Therefore $x_k \notin \mathrm{Anc}(x_i)$ implies
$x_k \notin \mathrm{Anc}(x_j)$, and in particular $x_k \notin K_j$.

Hence $K_i \subseteq X \setminus \{x_i, x_k\}$ and
$K_j \subseteq X \setminus \{x_j, x_k\}$.

\paragraph{Step~2: Constructing independent residuals without $x_k$.}
Choose $h_1^*$ and $h_2^*$ using the true structural functions
$f = f_i^1,\; g = g_i^1$ and $f = f_j^1,\; g = g_j^1$ respectively.

Since $K_i \subseteq X \setminus \{x_i, x_k\}$, the regression for $x_i$
on $K_i$ is a valid choice with exclusion set $S_1 = \{i, k\}$
(both $f_i^1$ and $g_i^1$ use only variables in
$K_i \subseteq X \setminus S_1$).  This gives:
\[
  h_1^*(i, \{i, k\})
  \;=\; \frac{x_i - f_i^1(K_i)}{g_i^1(K_i)}
  \;=\; \eta_i.
\]
Since $K_j \subseteq X \setminus \{x_j, x_k\}$, the same argument with
$S_2 = \{j, k\}$ gives:
\[
  h_2^*(j, \{j, k\})
  \;=\; \frac{x_j - f_j^1(K_j)}{g_j^1(K_j)}
  \;=\; \eta_j.
\]

\paragraph{Step~3: Independence via noise-set disjointness
(Lemma~\ref{lemma:independence_patterns}).}
Since $B(i,j) = 0$ (Premise~1), there is no UBP/UCP between $x_i$ and
$x_j$.  By the proof of Lemma~\ref{lemma:independence_patterns}, the
noise sets
\[
  N_i = \{\epsilon_l : v_l \in \mathrm{an}(Q_i) \cup Q_i \cup \{x_i\}\},
  \qquad
  N_j = \{\epsilon_l : v_l \in \mathrm{an}(Q_j) \cup Q_j \cup \{x_j\}\}
\]
satisfy $N_i \cap N_j = \emptyset$: any shared $\epsilon_l$ would imply a
UBP or UCP between $x_i$ and $x_j$ (Cases~1--3 of that proof),
contradicting $B(i,j) = 0$.  Since the exogenous noises $\epsilon_l$ are
mutually independent, $\eta_i = \phi_i(N_i)$ and $\eta_j = \phi_j(N_j)$
are functions of non-overlapping independent noise terms, giving
$\eta_i \cind \eta_j$.

Therefore $h_1^*(i,\{i,k\}) \cind h_2^*(j,\{j,k\})$, contradicting
Condition~\eqref{eq:cop-necessary1}.  By contrapositive,
$x_k \in \mathrm{Anc}(x_i)$.
\end{proof}
\begin{lemma}[Identifying ancestors on invisible paths for sets]
\label{lem:checkonpath1}
Consider distinct $x_i, x_j \in X$ and a set
$\{x_{k_1}, \ldots, x_{k_m}\} \subset X \setminus \{x_i, x_j\}$.
Suppose:
\begin{enumerate}
  \item $x_j$ is a visible parent of $x_i$
    (Definition~\ref{def:vis_parent}): $A(i,j) = 1$ and $B(i,j) = 0$.
  \item For each $q = 1,\ldots,m$, the pair $(x_i, x_{k_q})$ is invisible
    (Definition~\ref{def:invisible}):
    $\forall\, h_1, h_2 \in \mathcal{H}$:
    $h_1(i, \{i\}) \notcind h_2(k_q, \{k_q\})$.
\end{enumerate}
Under Assumptions~\ref{assumption:CFC}
and~\ref{assumption:residual_faithfulness}, suppose further that for each
$q = 1, \ldots, m$ the following holds:
\begin{equation}\label{eq:cop-necessary-set}
  \forall\, h_1, h_2 \in \mathcal{H} :\quad
  h_1(i, \{i, k_q\})
  \;\notcind\;
  h_2(j, \{j, k_q\}).
\end{equation}

\noindent Then each $x_{k_q}$ is an ancestor of $x_i$.
\end{lemma}

\begin{proof}
Fix any $q \in \{1,\ldots,m\}$. We prove the contrapositive: if
$x_{k_q} \notin \mathrm{Anc}(x_i)$, then independence is achievable
in~\eqref{eq:cop-necessary-set}. Proof same as before.

\end{proof}
\section{Simulation Details}
\label{app:datagen}

We describe the procedure used to generate one simulation trial.
All randomness is controlled by a single integer seed for full reproducibility.
Table~\ref{tab:datagen_params} lists all hyperparameters.

\subsection*{Step~1: Graph construction}

\paragraph{Observed skeleton.}
Direct edges among the $p$ observed variables $\{x_0,\ldots,x_{p-1}\}$ are
drawn from an Erd\H{o}s--R\'enyi DAG with edge probability $0.3$, following
\citet{maeda21a} Section~5.1 exactly.
For each ordered pair $(j, i)$ with $j < i$, the edge $x_j \to x_i$ is
included independently with probability $0.3$.
Using index order as the topological order guarantees acyclicity.

\paragraph{Hidden common causes (UBPs).}
For each of the $n_{\mathrm{cc}}$ hidden common causes $u_k$:
\begin{enumerate}
  \item Select a pair $(x_a, x_b)$ of observed variables that have no direct
        edge between them in either direction.
  \item Add directed edges $u_k \to x_a$ and $u_k \to x_b$.
        The node $u_k$ is a root (no parents of its own).
\end{enumerate}
When $u_k$ is marginalised out, the projected ADMG acquires the bidirected edge
$x_a \leftrightarrow x_b$, representing an Unobserved Backdoor Path (UBP).
Because $(x_a, x_b)$ was chosen to have no direct edge, the bow-free condition is satisfied for this pair.

\paragraph{Hidden intermediates (UCPs).}
For each of the $n_{\mathrm{int}}$ hidden intermediates $y_k$:
\begin{enumerate}
  \item Select an existing direct observed edge $x_j \to x_i$.
  \item \emph{Remove} the direct edge $x_j \to x_i$.
  \item Add $x_j \to y_k$ and $y_k \to x_i$.
\end{enumerate}
When $y_k$ is marginalised out, the projected ADMG acquires the bidirected
edge $x_j \leftrightarrow x_i$, representing an Unobserved Causal Path (UCP).
Removing the direct edge ensures bow-freeness.

\subsection*{Step~2: Data generation}

All variables are generated in topological order of the full latent DAG, so that every parent of $v_i$ has already been assigned values when $v_i$ is drawn.

\paragraph{Two-level generation.}
For each variable $x_i$, let $K_i$ denote its observed direct parents and $Q_i$ its hidden direct parents. We implement Equation~\eqref{eq: LSNM-UV} in two layers:
\begin{align}
  \eta_i   &= f_i^2(Q_i) + g_i^2(Q_i)\,\epsilon_i,
  \label{eq:layer2} \\
  x_i      &= f_i^1(K_i) + g_i^1(K_i)\,\eta_i,
  \label{eq:layer1}
\end{align}
where $\epsilon_i \overset{\mathrm{i.i.d.}}{\sim} \mathcal{N}(0,1)$.
Layer~\eqref{eq:layer2} aggregates the influence of all hidden parents into a single noise term $\eta_i$.
Layer~\eqref{eq:layer1} depends only on $K_i$, so hidden parents enter $x_i$ exclusively through $\eta_i$. Variables with no parents are pure noise: $x_i = \epsilon_i$.

\paragraph{Nonlinear functional form.}
Each structural component is assigned a nonlinear function family from Table~\ref{tab:func_families}. Our experiments use: $f^1 = \text{RBF}$, $g^1 = \text{tanh}$, $f^2 = \text{softplus}$, $g^2 = \text{logarithmic}$. For each parent $v_j$, the assigned family provides a randomly parameterised building block $\varphi(v_j)$ whose parameters are drawn independently per parent, per variable, and separately for location and scale.

\begin{table}[h]
  \centering
  \caption{Nonlinear function families used as building blocks $\varphi(v_j)$. All parameters are drawn independently per parent and per variable.}
  \label{tab:func_families}
  \small
  \begin{tabular}{@{}lll@{}}
    \toprule
    Family & Form $\varphi(v_j)$ & Parameters \\
    \midrule
    RBF & $a\exp\!\bigl(-b(v_j-c)^2\bigr)+d$ & $a{\sim}\mathrm{U}(1,3){\cdot}\{{\pm}1\}$, $b{\sim}\mathrm{U}(0.1,1)$, $c{\sim}\mathrm{U}(-2,2)$, $d{\sim}\mathrm{U}(-1,1)$ \\[3pt]
    Tanh & $a\tanh(bv_j+c)+d$ & $a{\sim}\mathrm{U}(1,3){\cdot}\{{\pm}1\}$, $b{\sim}\mathrm{U}(0.5,2)$, $c{\sim}\mathrm{U}(-2,2)$, $d{\sim}\mathrm{U}(-1,1)$ \\[3pt]
    Softplus & $a\log\!\bigl(1+e^{bv_j+c}\bigr)+d$ & $a{\sim}\mathrm{U}(0.5,2){\cdot}\{{\pm}1\}$, $b{\sim}\mathrm{U}(0.5,2)$, $c{\sim}\mathrm{U}(-2,2)$, $d{\sim}\mathrm{U}(-1,1)$ \\[3pt]
    Logarithmic & $a\,\mathrm{sign}(z)\log(|z|{+}1)+d$ & $z{=}bv_j{+}c$;\; $a{\sim}\mathrm{U}(1,3){\cdot}\{{\pm}1\}$, $b{\sim}\mathrm{U}(0.5,2)$, $c{\sim}\mathrm{U}(-2,2)$, $d{\sim}\mathrm{U}(-1,1)$ \\
    \bottomrule
  \end{tabular}
\end{table}
The location and scale functions are additive over parents:
\begin{align}
  f(K_i)
    &= \sum_{x_j \in K_i} \varphi_{f}(x_j),
  \label{eq:loc} \\
  g(K_i)
    &= \exp\!\left(
         \frac{\displaystyle\sum_{x_j \in K_i}
               \varphi_{g}(x_j)}
              {\mathrm{sd}\!\left(\sum_{x_j} \varphi_{g}\right)}
       \right),
  \label{eq:scale}
\end{align}
and analogously for $f_i^2, g_i^2$ with $Q_i$.
The log-scale sum in \eqref{eq:scale} is normalised by its empirical standard deviation before exponentiation to prevent numerical overflow; the result is clipped to $[0.1,\, 10]$.
Each variable is standardised to zero mean and unit empirical variance; $\eta_i$ is likewise standardised before being used in Layer~1.

The key structural difference from \citet{maeda21a} is the \emph{heteroscedastic} scale $g_i^1(K_i)$: in CAM-UV the noise is additive and homoscedastic ($g_i^1 \equiv 1$), whereas here the noise variance of $x_i$ varies with the values of its observed parents.

\subsection*{Step~3: ADMG projection}

After all data are generated the hidden variable columns are discarded.
The ground-truth adjacency matrix $A$ and bidirected matrix $B$ are read off
directly from the graph construction records:
$A(i,j)=1$ if there is a surviving direct edge $x_j \to x_i$ among observed
variables; $B(i,j)=1$ if the pair $(x_i,x_j)$ is the endpoint of a UBP or UCP.
By construction $A(i,j)B(i,j)=0$ for all $i,j$ (bow-free).

\subsection*{Step~4: Random column permutation}

A uniformly random permutation $\pi$ is applied to the columns of the data
matrix $X$ before it is returned.
The ground-truth matrices $A$ and $B$ are permuted consistently via
$A' = A[\pi,\pi]$, $B' = B[\pi,\pi]$.
This ensures the algorithm cannot exploit the coincidence that variable
indices equal the topological order of the generating DAG.

\subsection*{Implementation}

\textbf{LSNM-UV} uses exactly the CAM-UV algorithm of \citet{maeda21a} (Phases~1--3 from \texttt{lingam.CAMUV}, v1.12.2), replacing only the residual computation: instead of additive OLS residuals, we use LSNM residuals obtained via GAMLSS location-scale regression. After Phases~1--3, we apply the \texttt{checkVisible} procedure of \citet{pham2025} without modification to orient visible edges. No additional post-processing is performed beyond \texttt{checkVisible}.

\subsection*{Evaluation}

\paragraph{Methods and software.}
All methods receive the same $n \times p$ data matrix $X$ (with permuted columns) and significance level $\alpha = 0.01$.

\begin{itemize}
\item \textbf{LSNM-UV}. Uses exactly the \texttt{lingam.CAMUV} code (v1.12.2) with LSNM residuals, followed by the \texttt{checkVisible} procedure of \citet{pham2025}. Outputs a $p \times p$ adjacency matrix: $\mathrm{mat}[i,j] = 1$ denotes $x_j \to x_i$, $\mathrm{mat}[i,j] = \mathrm{NaN}$ denotes an invisible (bidirected) pair, and $\mathrm{mat}[i,j] = 0$ denotes no edge. We extract $\hat{A}$ by setting NaN entries to $0$, and $\hat{B}(i,j) = 1$ whenever $\mathrm{mat}[i,j]$ or $\mathrm{mat}[j,i]$ is NaN.

\item \textbf{CAM-UV}~\citep{maeda21a}. \texttt{lingam.CAMUV} (v1.12.2) with default HSIC independence test and \texttt{num\_explanatory\_vals}${}=3$. Output parsing identical to LSNM-UV.

\item \textbf{FCI}~\citep{spirtes2000,zhang2008}. From the \texttt{causal-learn} package (v0.1.4.5) with Fisher-$z$ test at $\alpha = 0.01$. We extract only \emph{definite} edges (no circle marks): directed $x_j \to x_i$ when $G[j,i] = -1$ and $G[i,j] = 1$; bidirected $x_i \leftrightarrow x_j$ when $G[i,j] = G[j,i] = 1$.

\item \textbf{BANG}~\citep{wang2020bang}. Source from \url{https://github.com/ysamwang/ngBap}, called via \texttt{rpy2} (v3.5.11, R v4.3.1). Run as \texttt{bang(X, K=3, level=0.01, restrict=1, testType="dhsic")} with the \texttt{dHSIC} R package (v2.2). $\hat{A}(i,j) = 1$ wherever $\texttt{dEdge}[i,j] \neq 0$; $\hat{B}(i,j) = 1$ from off-diagonal nonzero entries of \texttt{bEdge}.
\end{itemize}

\paragraph{Metrics.}
For directed edges, a true positive requires both position and direction to be correct: $\hat{A}(i,j) = A(i,j) = 1$. NaN entries are treated as $0$ for evaluation. Precision $= \mathrm{TP}/(\mathrm{TP}+\mathrm{FP})$, Recall $= \mathrm{TP}/(\mathrm{TP}+\mathrm{FN})$, F1 $= 2 \cdot \mathrm{Prec} \cdot \mathrm{Rec} / (\mathrm{Prec} + \mathrm{Rec})$.
For bidirected edges, the same definitions apply using $\hat{B}$ and $B$, evaluated on upper-triangle pairs only.

\paragraph{Error bars.}
The $400$ trials are split into four independent batches of $100$ trials each. Reported values are the mean across batch-level means; error bars show $\pm 1$ standard deviation across the four batch means.

\begin{table}[H]
  \centering
  \caption{Hyperparameters of the data-generation procedure.}
  \label{tab:datagen_params}
  \small
  \begin{tabular}{@{}llp{6.5cm}@{}}
    \toprule
    Parameter & Value & Description \\
    \midrule
    $p$                    & $10$              & Number of observed variables \\
    $n_{\mathrm{cc}}$      & $2$               & Hidden common causes (one UBP each) \\
    $n_{\mathrm{int}}$     & $2$               & Hidden intermediates (one UCP each) \\
    ER edge probability    & $0.3$             & Matches \citet{maeda21a} Sec.~5.1 \\
    Function families      & RBF / tanh / softplus / logarithmic & $f^1$ / $g^1$ / $f^2$ / $g^2$ \\
    $n$                    & $200$--$1000$     & Sample sizes evaluated (steps of $200$) \\
    Trials per setting     & $400$             & Independent random seeds \\
    \bottomrule
  \end{tabular}
\end{table}

\section{Broader Impacts}\label{appendix:broader_impacts}

This work is theoretical and methodological, with indirect societal impact. It may improve causal discovery from observational data in settings with latent variables and non-additive, heteroscedastic noise, which are common in areas such as biomedicine, economics, education, and environmental science. By relaxing additive-noise assumptions, the proposed approach can support more reliable scientific hypothesis generation and guide follow-up experiments.

The main risk is misuse or overinterpretation. The guarantees rely on assumptions such as acyclicity, faithfulness-type conditions, location-scale structure, and, for full ADMG recovery, bow-freeness. In finite samples, errors in regression or independence testing may lead to incorrect causal graphs. Such errors could be harmful in high-stakes domains if inferred causal relations are treated as definitive evidence for interventions or decisions. Missed latent confounding is a particular concern.

Accordingly, this method should be used as an exploratory tool, not as a standalone basis for high-stakes decisions. Practical use should include uncertainty quantification, sensitivity analysis, domain expertise, and, where possible, validation with interventional or prospective data. The experiments use synthetic data and do not involve human subjects or private information; the computational footprint is modest.

\end{document}